\documentclass[journal]{IEEEtran}

\usepackage{verbatim}
\usepackage{amssymb}
\usepackage{mathtools}
\usepackage{amsthm}
\usepackage{commath}
\usepackage{nth}
\usepackage{cite}
\usepackage{cleveref}
\usepackage{caption}
\usepackage{subcaption}

\usepackage[vlined,ruled]{algorithm2e}

\usepackage{amsmath,upgreek,graphicx,epsfig,color,amsfonts}

\graphicspath{{./fig/}}

\def\prob{\mathbb{P}}
\def\expt{\mathbb{E}}
\def\real{\mathbb{R}}

\def\natural{\mathbb{N}}
\def\var{\mathbf{Var}}
\def\indicator{\mathbf{1}}
\newcommand{\until}[1]{\{1,\dots, #1\}}
\newcommand{\subscr}[2]{#1_{\textup{#2}}}
\newcommand{\supscr}[2]{#1^{\textup{#2}}}
\newcommand{\setdef}[2]{\left\{#1 \; | \; #2\right\}}
\newcommand{\seqdef}[2]{\left\{#1\right\}_{#2}}

\newcommand{\union}{\operatorname{\cup}}

\newcommand{\ceil}[1]{\left\lceil #1 \right\rceil}
\newcommand{\floor}[1]{\left\lfloor #1  \right\rfloor}

\newcommand\oprocendsymbol{\hbox{$\square$}}
\newcommand\oprocend{\relax\ifmmode\else\unskip\hfill\fi\oprocendsymbol}

\def \mc {\mathcal}

\DeclareMathOperator{\Var}{Var}
\DeclareMathOperator{\sat}{sat}
\DeclareMathOperator{\sign}{sign}

\newtheorem{theorem}{Theorem}

\newtheorem{lemma}[theorem]{Lemma}

\newtheorem{remark}{Remark}

\newtheorem{assumption}{Assumption}

\newtheorem{fact}{Fact}

\renewcommand{\theenumi}{(\roman{enumi}}

\ifCLASSINFOpdf
\else
\fi

\begin{document}
\title{Nonstationary Stochastic Multiarmed Bandits: \\ UCB Policies and Minimax Regret
\thanks{This work was supported by NSF Award IIS-1734272}
}

\author{Lai Wei \hspace{1in} Vaibhav Srivastava% <-this % stops a space
	%\thanks{*This work was not supported by any organization}% <-this % stops a space
	\thanks{L. Wei and V. Srivastava are with the Department of Electrical and Computer Engineering. Michigan State University, East Lansing, MI 48823 USA.
		{\tt\small e-mail: weilai1@msu.edu; e-mail: vaibhav@egr.msu.edu }}%
}

\maketitle

\begin{abstract}
We study the nonstationary stochastic Multi-Armed Bandit (MAB) problem in which the distribution of rewards associated with each arm are assumed to be time-varying and the total variation in the expected rewards is subject to a variation budget. The regret of a policy is defined by the difference in the  expected cumulative rewards obtained using the policy and using an oracle that selects the arm with the maximum mean reward at each time. We characterize the performance of the proposed policies in terms of the worst-case regret, which is the supremum of the regret over the set of reward distribution sequences satisfying the variation budget. 
% In the problem, a policy is evaluated by the regret, i.e., the gap in expected cumulative rewards comparing with a oracle who selects the arm with maximum mean reward at every time instance. 
% The objective is to design policies to minimize the worst-case regret,  which is the regret with respect to the worst possible choice of  
We  extend Upper-Confidence Bound (UCB)-based policies with three different approaches, namely, periodic resetting, sliding observation window and discount factor and show that they are order-optimal with respect to the minimax regret, i.e., the minimum worst-case regret achieved by any policy. We also relax the sub-Gaussian assumption on reward distributions and develop  robust versions the proposed polices that can handle heavy-tailed reward distributions and maintain their performance guarantees. 
\end{abstract}


% \fi
\IEEEpeerreviewmaketitle

\section{Introduction}
\IEEEPARstart{U}{ncertainty} and nonstationarity of the environment are two of the major barriers in decision-making problems across scientific disciplines, including engineering, economics, social science, neuroscience, and ecology.
An efficient strategy in such environments requires balancing several tradeoffs, including  \emph{exploration-versus-exploitation}, i.e., choosing between the most informative and the empirically most rewarding alternatives, and \emph{remembering-versus-forgetting}, i.e., using more but possibly outdated information or using less but recent information. 

% A decision-making agent in an uncertain environment needs to choose between the most informative option and the empirically most rewarding alternative, which is referred to as the  tradeoff. When facing nonstationarity, a compromise needs to be made between using more but possibly outdated information and less but updated information, which is the \emph{remembering-versus-forgetting} tradeoff. The capability to handle these two tradeoffs in many circumstances determines the performance.

The stochastic MAB problem is a canonical formulation of the exploration-versus-exploitation tradeoff. 
% It was historically proposed in~\cite{WRT:33} to model the clinical trials, where different treatments with unknown efficacy are viewed as stochastic arms. 
In an MAB problem, an agent selects one from $K$ options at each time and receives a reward associated with it. The reward sequence at each option is assumed to be an unknown i.i.d random process.
%Other than clinical trials, 
%
The MAB formulation has been applied in many scientific and technological areas. For example,  it is used for opportunistic spectrum access in communication networks, wherein the arm models the availability of a channel~\cite{alaya2008dynamic, anandkumar2011distributed}. In MAB formulation of online learning for demand response\cite{LI2020109015,7447007}, an aggregator calls upon a subset of users (arms) who have an unknown response to the request to reduce their loads. 
% change their demand and adjust the total load close to a target value. 
MAB formulation has also been used in robotic foraging and surveillance~\cite{JRK-AK-PT:78,VS-PR-NEL:13, VS-PR-NEL:14, baykal2017persistent} and acoustic relay positioning for underwater communication~\cite{MYC-JL-FSH:13}, wherein the information gain at different sites is modeled as rewards from arms. Besides, contextual bandits are widely used in recommender systems~\cite{agarwal2009online, li2010contextual}, wherein the acceptation of a recommendation corresponds to the rewards from an arm. The stationarity assumption in classic MAB problems limits their utility in these applications since channel usage, robot working environment and people's preference are inherently uncertain and evolving. In this paper, we relax this assumption and study non-stationary stochastic MAB problems.

Robbins~\cite{robbins1952} formulated the objective of the stochastic MAB problem as minimizing the \emph{regret}, that is, the loss in expected cumulative rewards caused by failing to select the best arm every time. In their seminal work, Lai and Robbins~\cite{TLL-HR:85}, followed by Burnetas and Katehakis~\cite{burnetas1996optimal}, established a logarithm \emph{problem-dependent} asymptotic lower bound on the regret achieved by any policy, which has a leading constant determined by the underlying reward distributions.  A general method of constructing UCB rules for parametric families of reward distributions is also presented in~\cite{TLL-HR:85}, and the associated policy is shown to attain the logarithm lower bound. Several subsequent UCB-based algorithms~\cite{PA-NCB-PF:02,kl-ucb} with efficient finite time performance have been proposed.

% . Among these, the classic UCB1 algorithm~\cite{PA-NCB-PF:02} achieves logarithm finite-time performance for bounded rewards, and Kullback-Leibler-UCB (KL-UCB)~\cite{kl-ucb} yields improved performance compared with UCB1. Other efficient non-UCB policies include Thompson Sampling (TS)~\cite{WRT:33,thompson_agrawal,thompson_kaufmann} and Minimum Empirical Divergence policy~\cite{honda2010asymptotically}.

The adversarial MAB~\cite{nonstochastic} is a paradigmatic nonstationary problem. In this model, the bounded reward sequence at each arm is arbitrary.
%subjecting to no probability distribution. 
The performance of an policy is evaluated using the \emph{weak regret}, which is the difference in the cumulated reward of a policy compared against the best single action policy.  A $\Omega(\sqrt{KT})$ lower bound on the weak regret and a near-optimal policy Exp$3$ is also presented in~\cite{nonstochastic}. While being able to capture nonstationarity, the generality of the reward model in  adversarial MAB makes the investigation of globally optimal policies very challenging. 

% this framework fails to fit many applications where the environment does not possess the adversarial phenomenon. 

The nonstationary stochastic MAB can be viewed as a compromise between stationary stochastic MAB and adversarial MAB. It maintains the stochastic nature of the reward sequence while allowing some degree of nonstationarity in reward distributions. Instead of the weak regret analyzed in adversarial MAB, a strong notion of regret
% comparing with the single best arm that is selected at each time, the regret here is 
defined with respect to the best arm at each time step is studied in these problems. A broadly studied nonstationary problem is \emph{piecewise stationary} MAB, wherein the reward distributions are piecewise stationary. To deal with the remembering-versus-forgetting tradeoff, the idea of using discount factor to compute the UCB index is proposed in~\cite{kocsis2006discounted}. Garivier and Moulines~\cite{AG-EM:08} present and analyze Discounted UCB (D-UCB) and Sliding-Window UCB (SW-UCB), in which they compute the UCB using discounted sampling history and recent sampling history, respectively. 
% contained in a sliding window. 
They pointed out that if the number of change points $N_T$ is available, both algorithms can be tuned to achieve regret close to the $\Omega(\sqrt{K N_T T})$ regret lower bound. In our earlier work~\cite{LW-VS:17i}, the near optimal regret is achieved using deterministic sequencing of explore and exploit with limited memory. Other works handle the change of reward distributions in an adaptive manner by adopting change point detection techniques~\cite{hartland2007change,liu2018change,besson2019generalized,cao2019nearly,mellor2013thompson}. 
% In~\cite{hartland2007change}, the change points are  actively detected with Page-Hinkley test. Other change point detection techniques such as cumulative sum (CUMSUM) and Generalized Likelihood Ration (GLR) test are used in subsequent work to design CUMSUM-UCB~\cite{liu2018change}, GLR-klUCB~\cite{besson2019generalized}, M-UCB~\cite{cao2019nearly} and EXP3.R~\cite{allesiardo2017non} algorithms. These algorithms either need the knowledge of $N_T$ to tune the parameter or having regret upper bounds with strong dependency on $N_T$. Change point detection has also been implemented together with TS to design Change-Point TS~\cite{mellor2013thompson}.

% Until very recently, two parameter-free policies \textsc{AdSwitch}~\cite{auer2019adaptively} and \textsc{Ada}-ILTCB\textsuperscript{+}~\cite{chen2019new} for contextual MAB are proved to have $\tilde{O}(\sqrt{K N_T T})$ regret. 

A more general nonstationary problem is studied in~\cite{Rexp3}, wherein the cumulative maximum variation in mean rewards is subject to a variation budget $V_T$. 
% So both moderate and drastic variations in the environment are captured. In the same paper, they established 
Additionally, the authors in~\cite{Rexp3} establish a $\Omega((KV_T)^{\frac{1}{3}}T^{\frac{2}{3}})$ minimax regret lower bound  and propose the Rexp$3$ policy. In their subsequent work~\cite{besbes2019optimal}, they tune Exp$3$.S policy from~\cite{nonstochastic} to achieve near optimal worst-case regret. Discounted Thomson Sampling (DTS)~\cite{raj2017taming} has also been shown to have good experimental performance within this general framework. However, we are not aware of any analytic regret bounds for the DTS algorithm. 

% Cheung et al.~\cite{pmlr-v89-cheung19b} study extended the variation budget idea into linear stochastic bandits and solved the problem by modifying the SW-UCB algorithm. 
% If $V_T$ is available, they are able to achieve $\tilde{O}((KV_T)^{\frac{1}{3}}T^{\frac{2}{3}})$ worst-case regret. When $V_T$ is not available, an upper-level MAB is constructed to facilitate selecting sliding window width. A suboptimal $\tilde{O}((KV_T)^{\frac{1}{4}}T^{\frac{3}{4}})$ regret is attained. 
% In~\cite{chen2019new}, the variation budget is with respect to the total variation in probability distributions. In restless bandits~\cite{slivkins2008adapting, ortner2012regret}, the variation in reward distributions are assumed to be generated by some random processes.

In this paper, we follow the more general nonstationary stochastic MAB formulation in~\cite{Rexp3} and design UCB-based policies that achieve efficient performance in environments with sub-Gaussian as well as heavy-tailed rewards. We focus on UCB-based policies instead of EXP$3$-type policies because EXP$3$-type policies require bounded rewards and  have large variance in cumulative rewards~\cite{nonstochastic}. Additionally, by using robust mean estimator, UCB-based policies for light-tailed rewards can be extended to handle heavy-tailed reward distributions, which exist in many domains such as social networks~\cite{albert2002statistical} and financial markets~\cite{vidyasagar2010law}. The major contributions of this work are: 

\begin{itemize}
    \item Assuming the variation density $V_T /T$ is known, we extend MOSS~\cite{MOSS} to design Resetting MOSS (R-MOSS) and Sliding-Window MOSS (SW-MOSS). Also, we show D-UCB can be tuned to solve the problem.
    
    \item With rigorous analysis, we show that R-MOSS and SW-MOSS achieve the exact order-optimal minimax regret and D-UCB achieves near-optimal worst-case regret.
    
    \item We relax the bounded or sub-Gaussian assumption on the rewards required by Rexp$3$ and SW-UCB and design policies robust to heavy-tailed rewards. We show the theoretical guarantees on the worst-case regret can be maintained by the robust policies.
\end{itemize}

% {\color{red} At last, we extend the improved UCB~\cite{auer2010ucb} to design a parameter-free policy enjoying $\tilde{O}((KV_T)^{\frac{1}{3}}T^{\frac{2}{3}})$ worst-case regret.}

The remainder of the paper is organized as follows. We formulate nonstationary stochastic MAB with variation budget in Section~\ref{sec: problem formulation} and review some preliminaries in Section~\ref{sec: minimax regret}. In Section~\ref{sec: 3 policies}, we present and analyze three UCB policies: R-MOSS, SW-MOSS and D-UCB. We present and analyze algorithms for nonstationary heavy-tailed bandit in Section~\ref{sec: heavy-tailed}.
% {\color{red} We propose and analyze a simple parametric-free policy in section~\ref{sec: parameter-free}.}
We complement the theoretical results with numerical illustrations in Section~\ref{sec: simulation} and conclude this work in Section~\ref{sec: conclusion}.

\section{Problem Formulation} \label{sec: problem formulation}

We consider a nonstationary stochastic MAB problem with $K$ arms and a horizon length $T$. Let $\mathcal{K} := \until{K}$ be the set of arms and $\mathcal{T} :=\until{T}$ be the sequence of time slots. The reward sequence $\seqdef{X_t^k}{t\in\mathcal{T}}$ for each arm $k \in \mathcal{K}$ is composed of independent samples from potentially time-varying probability distribution function sequence $f_{\mathcal{T}}^k := \seqdef{f_t^k (x)}{t\in \mathcal{T}}$. We refer to the set $ \mathcal{F}_{ T}^{\mathcal{K}} = \setdef{f_{\mc T}^k}{k \in \mathcal{K}}$ containing reward distribution sequences at all arms as the \emph{environment}. Let $\mu_t^k = \expt [X_t^k]$. 
Then, the \emph{total variation} of $\mathcal{F}_T^{\mathcal{K}}$ is defined by
\begin{equation}\label{def: variation}
v \big(\mathcal{F}_T^{\mathcal{K}}\big) : = \sum_{t=1}^{T-1} \sup_{k \in \mathcal{K}} \: \abs{\mu_{t+1}^k-\mu_t^k},
\end{equation}
which captures the non-stationarity of the environment. We focus on the class of non-stationary environments that have the total variation within a \emph{variation budget} $V_T \geq 0$ which is defined by
\[\mathcal{E}(V_T,T,K) := \big \{\mathcal{F}_T^{\mathcal{K}} \; | \; v\big(\mathcal{F}_T^{\mathcal{K}}\big) \leq V_T \big \}.\]

At each time slot $t \in \mathcal{T}$, a decision-making agent selects an arm $\varphi_t \in \mathcal{K}$ and receives an associated random reward $ X_t^{\varphi_t}$. The objective is to maximize the expected value of the \emph{cumulative reward} $S_T :=\sum_{t=1}^{T} X_t^{\varphi_t}$. We assume that  $\varphi_t$ is selected based upon past observations $\{X_s^{\varphi_s}, \varphi_s\}_{s=1}^{t-1}$ following some policy $\rho$. Specifically, $\rho$ determines the conditional distribution
\[\prob^\rho\left(\varphi_t=k \;|\; \{X_s^{\varphi_s}, \varphi_s\}_{s=1}^{t-1} \right)\]
at each time $t \in \until{T-1}$. If $\prob^\rho\left( \cdot \right)$ takes binary values, we call $\rho$ deterministic; otherwise, it is called stochastic. 

Let the expected reward from the best arm at time $t$ be $\mu_{t}^*=\max_{k \in \mathcal{K}} \mu_t^k. $
Then, maximizing the expected cumulative reward is equivalent to minimizing the \emph{regret} defined by
\[R^\rho_T := \sum_{t=1}^T \mu_{t}^* - \expt^\rho[ S_T] = \expt^\rho \Bigg[\sum_{t=1}^T  \mu_{t}^*- \mu_t^{\varphi_t}\Bigg],\]
where the expectation is with respect to different realization of $\varphi_t$ that depends on obtained rewards through policy $\rho$. 

Note that the performance of a policy $\rho$ differs with different $\mathcal{F}_T^{\mathcal{K}} \in \mathcal{E}(V_T,T,K)$. For a fixed variation budget $V_T$ and a policy $\rho$, the \emph{worst-case regret} is the regret with respect to the worst possible choice of environment, i.e.,
\[
\subscr{R}{worst}^\rho(V_T,T,K) =\sup_{\mathcal{F}_T^{\mathcal{K}} \in \mathcal{E}(V_T,T,K)} \: R_T^\rho.\]
In this paper, we aim at designing policies to minimize the worst-case regret. The optimal worst-case regret achieved by any policy is called the \emph{minimax regret}, and is defined by 
\[ \inf_{\rho } \sup_{\mathcal{F}_T^{\mathcal{K}} \in \mathcal{E}(V_T,T,K)} \: R_T^\rho. \]
We will study the nonstationary MAB problem under the following two classes of reward distributions:
\begin{assumption}[Sub-Gaussian reward]
	For any $k\in \mathcal{K}$ and any $t\in \mathcal{T}$, distribution $f_t^k (x)$ is $1/2$ sub-Gaussian, i.e.,
	\[ \forall \lambda \in \real: \expt \left[ \exp (\lambda(X_t^k - \mu))  \right] \leq \exp \left( \frac{\lambda^2}{8}\right). \]
	Moreover, for any arm $k\in \mathcal{K}$ and any time $t\in\mathcal{T}$,  $\expt\left[X_t^k\right] \in [a,a+b]$, where $a\in\real$ and $b>0$.
\end{assumption}
\noindent

\begin{assumption}[Heavy-tailed reward]\label{ass: heavy-tailed}
	For any arm $k \in \mathcal{K}$ and any time $t \in \mathcal{T}$, $\expt \left[(X_t^k)^2\right] \leq 1$.
\end{assumption}

\section{Preliminaries} \label{sec: minimax regret}

In this section, we review existing minimax regret lower bounds and minimax policies from literature. These results apply to both sub-Gaussian and heavy-tailed rewards. The discussion is made first for $V_T=0$. Then, we show how the minimax regret lower bound for $V_T=0$ can be extended to establish the minimax regret lower bound for $V_T > 0$. To this end, we review two UCB algorithms for the stationary stochastic MAB problem: UCB1 and MOSS. In the later sections, they are extended to design a variety of policies to match with the minimax regret lower bound for $V_T > 0$.

\subsection{Lower Bound for Minimax Regret when $V_T=0$} \label{lower bound v=0}

In the setting of $V_T=0$, for each arm $k \in \mathcal{K}$, $\mu_t^k$ is identical for all $t \in \mathcal{T}$. In stationary stochastic MAB problems, the rewards from each arm $k\in \mathcal{K}$ are independent and identically distributed, so they belong to the environment set $\mathcal{E}(0,T,K)$. 
% Therefore, we can borrow the minimax regret lower bound there. 
According to~\cite{auer1995gambling}, if $V_T =0$, the minimax regret is no smaller than $1/20 \sqrt{KT}$. This result is closely related to the standard logarithmic lower bound on regret for stationary stochastic MAB problems as discussed below. Consider a scenario in which there is a unique best arm and all other arms have identical mean rewards such that the gap between optimal and suboptimal mean rewards is $\Delta$. From~\cite{mannor2004sample}, for such a stationary stochastic MAB problem
% for any policy $\rho$, there exists a stationary stochastic MAB problem with gaps in mean all equal to $\Delta$ such that
\begin{equation}\label{eq: stationary lb}
	R_T^\rho \geq C_1 \frac{K}{\Delta} \ln \Big(\frac{T \Delta^2}{K}\Big) + C_2 \frac{K}{\Delta},
\end{equation}
for any policy $\rho$, where $C_1$ and $C_2$ are some positive constants.
It needs to be noted that for $\Delta = \sqrt{K/T}$, the above lower bound becomes $C_2\sqrt{KT}$, which matches with the lower bound $1/20 \sqrt{KT}$.

\subsection{Lower Bound for Minimax Regret when $V_T>0$}

In the setting of $V_T>0$, we recall here the minimax regret lower bound for nonstationary stochastic MAB problems.

\begin{lemma}[Minimax Lower Bound: $V_T>0$~\cite{Rexp3}] \label{lemma:nonstationary lower bound}
	For the non-stationary MAB problem with $K$ arms, time horizon $T$ and variation budget $V_T \in [1/K, T/K]$, %it satisfies that
	\[\inf_{\rho }\sup_{\mathcal{F}_T^\mathcal{K} \in \mathcal{E}(V_T,T,K)}  R_T^\rho \geq C(KV_T)^{\frac{1}{3}}T^{\frac{2}{3}},\]
	where $C\in  \real_{> 0}$ is some constant.
\end{lemma}
To understand this lower bound, consider the following non-stationary environment. The horizon $\mathcal{T}$ is partitioned into epochs of length $\tau = \big\lceil{ K^\frac{1}{3} ({T/V_T})^{\frac{2}{3}}} \big \rceil $. In each epoch, the reward distribution sequences are stationary and all the arms have identical mean rewards except for the unique best arm. Let the gap in the mean be $\Delta = \sqrt{K/\tau}$. The index of the best arm switches at the end of each epoch following some unknown rule. So, the total variation is no greater than $\Delta T/ \tau$, which satisfies the variation budget $V_T$. Besides, for any policy $\rho$, we know from~\eqref{eq: stationary lb} that worst case regret in each epoch is no less than $C_2\sqrt{K\tau}$. Summing up the regret over all the epochs, minimax regret is lower bounded by $T/ \tau \times C_2\sqrt{K\tau} $, which is consistent with Lemma~\ref{lemma:nonstationary lower bound}.

	\subsection{UCB Algorithms in Stationary Environments}\label{ucb-background}
The family of UCB algorithms uses the principle called optimism in the face of uncertainty. In these policies, at each time slot, a UCB index which is a statistical index composed of both mean reward estimate and the associated uncertainty measure is computed at each arm, and the arm with the maximum UCB is picked. Within the family of UCB algorithms, two state-of-the-art algorithms for the stationary stochastic MAB problems are UCB$1$~\cite{PA-NCB-PF:02} and MOSS~\cite{MOSS}. Let $n_k(t)$ be the number of times arm $k$ is sampled until time $t-1$, and $ \hat \mu_{k,n_k(t)}$ be the associated empirical mean. Then, UCB$1$ computes the UCB index for each arm $k$ at time $t$ as 
\[\supscr{g_{k,t}}{UCB1} = \hat \mu_{k,n_k(t)} + \sqrt{\frac{ 2 \ln t}{n_k(t)}}.\]
It has been proved in~\cite{PA-NCB-PF:02} that, for the stationary stochastic MAB problem, UCB1 satisfies
\[ \supscr{R_T}{UCB1} \leq 8 \sum_{k:\Delta_k > 0} \frac{\ln T}{\Delta_k} + \left(1+\frac{\pi^2}{3}\right) \sum_{k=1}^K \Delta_k ,\]
where $\Delta_k$ is the difference in the mean rewards from arm $k$ and the best arm. In~\cite{MOSS}, a simple variant of this result is given by selecting values for $\Delta_k$ to maximize the upper bound, resulting in
\[\sup_{\mathcal{F}_T^\mathcal{K} \in \mathcal{E}(0,T,K)} \supscr{R_T}{UCB1} \leq 10\sqrt{(K-1)T(\ln T)}.\]

Comparing this result with the lower bound on the minimax regret discussed in Section~\ref{lower bound v=0}, there exists an extra factor $\sqrt{\ln T}$. This issue has been resolved by the MOSS algorithm. With prior knowledge of horizon length $T$, and the UCB index for MOSS is expressed as
\[\supscr{g_{k,t}}{MOSS} = \hat \mu_{k,n_k(t)} + \sqrt{\frac{ \max  \Big(\ln \Big(\frac{T}{K n_k(t)}\Big), 0\Big)}{n_k(t)}}.\]
We now recall the worst-case regret upper bound for MOSS.

\begin{lemma}[Worst-case regret upper bound for MOSS~\cite{MOSS}]\label{theorem:MOSS bound}
	For the stationary stochastic MAB problem ($V_T=0$), the worst-case regret of the MOSS algorithm satisfies
	\[\sup_{\mathcal{F}_T^\mathcal{K} \in \mathcal{E}(0,T,K)} R_T^{\text{MOSS}} \leq 49\sqrt{KT}.\]
\end{lemma}

\section{UCB Algorithms for Sub-Gaussian Nonstationary Stochastic MAB Problems} \label{sec: 3 policies}
In this section,  we extend UCB$1$ and MOSS to design nonstationary UCB policies for scenarios with $V_T>0$. Three different techniques are employed, namely periodic resetting, sliding observation window and discount factor, to deal with the remembering-forgetting tradeoff. The proposed algorithms are analyzed to provide guarantees on the worst-case regret. We show their performances match closely with the lower bound in Lemma~\ref{lemma:nonstationary lower bound}.

% The following notations about the nonstationary stochastic MAB problem will be used. 
The following notations are used in later discussions. Let $N=\ceil{T/\tau}$, for some $\tau \in \until{T}$,  and let $\{\mathcal{T}_1,\ldots, \mathcal{T}_N\}$ be a partition of time slots $\mathcal{T}$, where each epoch $\mc T_i$ has length $\tau$ except possibly $\mathcal{T}_N$. In particular,
\[ \mathcal{T}_i = \Big\{1+(i-1)\tau \: , \ldots , \:\min \left( i\tau, T \right ) \Big\}, \; i \in \until{N}.\]
Let the maximum mean reward within $\mathcal{T}_i$ be achieved at time $\tau_i \in \mathcal{T}_i$ and arm $\kappa_i$, i.e.,	
$\mu^{\kappa_i}_{\tau_i} = \max_{t \in \mathcal{T}_i} \; \mu_t^*$.
We define the variation within $\mathcal{T}_i $ as
\[ v_i := \sum_{t \in \mathcal{T}_i} \:  \sup_{k \in \mathcal{K}} \: \abs{\mu_{t+1}^k-\mu_t^k},\]
where we trivially assign $\mu_{T+1}^k=\mu_T^k$ for all $k \in \mathcal{K}$. Let $\indicator \left\{\cdot\right\}$ denote the indicator function and $\abs{\cdot}$ denote the cardinality of the set, if its argument is a set, and the absolute value if its argument is a real number.

\subsection{Resetting MOSS Algorithm}
Periodic resetting is an effective technique to preserve the freshness and authenticity of the information history. It has been employed in~\cite{Rexp3} to modify Exp$3$ to design Rexp$3$ policy for nonstationary stochastic MAB problems. We extend this approach to MOSS and propose nonstationary policy Resetting MOSS (R-MOSS). In R-MOSS, after every $\tau$ time slots, the sampling history is erased and MOSS is restarted.  The pseudo-code is provided in Algorithm~\ref{algo:r-moss} and the performance in terms of the worst-case regret for is established below.

\IncMargin{.3em}
% \restylealgo{boxed}
% \linesnumbered 
\begin{algorithm}[t]
	{\footnotesize
		\SetKwInOut{Input}{  Input}
		\SetKwInOut{Set}{  Set}
		\SetKwInOut{Title}{Algorithm}
		\SetKwInOut{Require}{Require}
		\SetKwInOut{Output}{Output}
		
		{
			
			\Input{$V_T \in \real_{\ge 0}$ and $T\in \natural$}
			
			\Set{$\tau = \Big \lceil { K^\frac{1}{3} \left(T / V_T \right)^{\frac{2}{3}}} \Big \rceil$}
			
		}
		
		\medskip		
		%		\Set{$\displaystyle \lambda= \begin{cases} 
		%			\frac{1}{1-\alpha}, & \mbox{for }  \alpha<1,\\
		%			+\infty , & \mbox{for }  \alpha=1, 
		%			\end{cases}
		%			\, \text{and }\,\xi = 1+\alpha;$}
		%		
		%		\medskip
		
		\Output{sequence of arm selection}
		
		\medskip
		
		\nl \While{$t \leq T$}{
						
			\smallskip
			\nl \If{$\mod(t,\tau) = 0$}{
				\smallskip
				
				\nl
				Restart the MOSS policy;
				\smallskip
				
				}

		}

		% wait for next sample and go to step~\algostep{1}
		
		%  \nocaptionofalgo
		\caption{\textit{R-MOSS}}
		\label{algo:r-moss}}
\end{algorithm} 
\DecMargin{.3em}

\begin{theorem}\label{theorem:R-MOSS bound}
	For the sub-Gaussian nonstationary MAB problem with $K$ arms, time horizon $T$, variation budget $V_T>0$, and $\tau = \Big \lceil { K^\frac{1}{3} \left(T / V_T \right)^{\frac{2}{3}}} \Big \rceil$, the worst case regret of R-MOSS satisfies
	\[\sup_{\mathcal{F}_T^\mathcal{K} \in \mathcal{E}(V_T,T,K)} \supscr{ R_T }{R-MOSS} \in  \mc O((KV_T)^{\frac{1}{3}} T^{\frac{2}{3}}).\]
% 	for some constant $C \in \real_{> 0}$.
\end{theorem}

\begin{proof}[Sketch of the proof]
	Note that one run of MOSS takes place in each epoch. For epoch $\mathcal{T}_i$, define the set of \emph{bad arms} for R-MOSS by
	\begin{equation}\label{def: Rbad}
		\supscr{\mathcal{B}_i}{R} := \setdef{ k \in \mathcal{K} }{ \mu_{\tau_i}^{\kappa_i} - \mu_{\tau_i}^k \geq 2 v_i }.
	\end{equation}
	Notice that for any $t_1, t_2 \in \mathcal{T}_i$,
	\begin{equation}\label{ineq: variation}
	\abs{\mu_{t_1}^k - \mu_{t_2}^k} \leq v_i , \quad \forall k \in \mathcal{K}.
	\end{equation}
	Therefore, for any $t\in \mathcal{T}_i$, we have
	\begin{align*}
	\mu_{t}^* - \mu_{t}^{\varphi_t} &\leq  \mu_{\tau_i}^{\kappa_i} - \mu_{t}^{\varphi_t} \leq \mu_{\tau_i}^{\kappa_i} - \mu_{\tau_i}^{\varphi_t} + v_i.
	\end{align*}
	Then, the regret from $\mathcal{T}_i$ can be bounded as the following,	
	\begin{align}
	\expt \bigg [ \sum_{t\in \mathcal{T}_i}  \mu_t^* - \mu_t^{\varphi_t} \bigg] &\leq \abs{\mathcal{T}_i} v_i +\expt \bigg [ \sum_{t\in \mathcal{T}_i}  \mu_{\tau_i}^{\kappa_i} - \mu_{\tau_i}^{\varphi_t} \bigg] \nonumber\\
	&\leq 3 \abs{\mathcal{T}_i} v_i + S_i, \label{ineq: regret_epoch}
	\end{align}
	where $\displaystyle S_i = \expt \bigg [ \sum_{t\in \mathcal{T}_i} \sum_{k \in \supscr{\mathcal{B}_i}{R}} \indicator \left\{\varphi_t = k \right\} \left(\mu_{\tau_i}^{\kappa_i} - \mu_{\tau_i}^{\varphi_t} - 2v_i \right) \bigg]$.
	
	Now, we have decoupled the problem, enabling us to the generalize the analysis of MOSS in stationary environment~\cite{MOSS} to bound $S_i$. We will only specify the generalization steps and skip the details for brevity.

 First notice inequality~\eqref{ineq: variation} indicates that for any $k \in \supscr{\mathcal{B}_i}{R}$ and any $t\in \mathcal{T}_i$,
	\[ \mu_{t}^{\kappa_i} \geq \mu_{\tau_i}^{\kappa_i} - v_i \text{ and } \mu_{t}^k \leq \mu_{\tau_i}^{k} + v_i.\]
	So, at any $t \in \mathcal{T}_i$, $\hat\mu_{{\kappa_i},n_{\kappa_i}(t)} $ concentrate around a value no smaller than $\mu_{\tau_i}^{\kappa_i} - v_i$, and  $\hat \mu_{k,n_k(t)}$ concentrate around a value no greater than $\mu_{\tau_i}^{k} + v_i$ for any $k \in \supscr{B_i}{R}$. Also $\mu_{\tau_i}^{\kappa_i} - v_i \geq \mu_{\tau_i}^{k} + v_i$ due to the definition in~\eqref{def: Rbad}. 
	
	In the analysis of MOSS in stationary environment~\cite{MOSS}, the UCB of each suboptimal arm is compared with the best arm and each selection of suboptimal arm $k$ contribute $\Delta_k$ in regret. Here, we can apply a similar analysis by comparing the UCB of each arm $k \in \supscr{B_i}{R}$ with $\kappa_i$ and each selection of arm $k \in \supscr{B_i}{R}$ contributes $(\mu_{\tau_i}^{\kappa_i} - v_i) - (\mu_{\tau_i}^{k} + v_i)$ in $S_i$. Accordingly, we borrow the upper bound in Lemma~\ref{theorem:MOSS bound} to get $S_i \leq 49\sqrt{K \abs{\mathcal{T}_i}}$.

	 Substituting the upper bound on $S_i$ into~\eqref{ineq: regret_epoch} and summarizing over all the epochs, we conclude that 
	\[ \sup_{\mathcal{F}_T^\mathcal{K} \in \mathcal{E}(V_T,T,K)} \supscr{ R_T }{R-MOSS} \leq 3\tau V_T + \sum_{i=1}^{N} 49\sqrt{K \tau},\]
	which implies the theorem. 
\end{proof}

The upper bound in Theorem~\ref{theorem:R-MOSS bound} is in the same order as the lower bound in Lemma~\ref{lemma:nonstationary lower bound}. So, the worst-case regret for R-MOSS is order optimal.

\subsection{Sliding-Window MOSS Algorithm }

We have shown that periodic resetting coarsely adapts the stationary policy to a nonstationary setting. However, it is inefficient to entirely remove the sampling history at the restarting points and the regret accumulates quickly close to these points. In~\cite{AG-EM:08}, a sliding observation window is used to erase the outdated information smoothly and more efficiently utilize the  information history. The authors proposed the SW-UCB algorithm that intends to solve the MAB problem with piece-wise stationary mean rewards. We show that a similar approach can also deal with the general nonstationary environment with a variation budget. In contrast to SW-UCB, we integrate the sliding window technique with MOSS instead of UCB1 and achieve the order optimal worst-case regret.

Let the sliding observation window at time $t$ be $\mathcal{W}_t :=\left\{\min(1,t-\tau),\ldots, t-1 \right \}$. Then, the associated mean estimator is given by
\[\hat \mu_{n_k(t)}^k \! = \! \frac{1}{n_k(t)}\! \sum_{s\in \mathcal{W}_t} \!\! X_s \indicator\{ \varphi_s=k \}, \,\, n_k(t) = \! \sum_{s\in \mathcal{W}_t} \!\! \indicator {\{ \varphi_s=k \}}. \]
For each arm $k \in \mathcal{K}$, define the UCB index for SW-MOSS by
\[g_{t}^k = \hat \mu_{n_k(t)}^k + c_{n_k(k)}, \,\, c_{n_k(t)} = \sqrt{\eta \frac{ \max  \Big(\ln \Big(\frac{\tau}{K n_k(t)}\Big), 0\Big)}{n_k(t)}}, \]
where $\eta>1/2$ is a tunable parameter. With these notations, SW-MOSS is defined in Algorithm \ref{algo:sw-moss}. To analyze it, we will use the following concentration bound for sub-Gaussian random variables.

\IncMargin{.3em}
% \restylealgo{boxed}
% \linesnumbered 
\begin{algorithm}[t]\label{sliding-window MOSS}
	{\footnotesize
		\SetKwInOut{Input}{  Input}
		\SetKwInOut{Set}{  Set}
		\SetKwInOut{Title}{Algorithm}
		\SetKwInOut{Require}{Require}
		\SetKwInOut{Output}{Output}
		
		{
%		\it For abruptly-changing environment} \\   
		
		\Input{$V_T \in \real_{> 0}$, $T\in \natural$ and $\eta>1/2$}
		
		\Set{$\tau = \Big \lceil { K^\frac{1}{3} \left(T/V_T \right) ^{\frac{2}{3}}} \Big \rceil$}
	
		}
		
		\medskip

		\Output{sequence of arm selection}
				
		\medskip
		
		\nl Pick each arm once.
		\smallskip
		
		\nl \While{$t \leq T$}{	
			\smallskip		
			{
				 Compute statistics within $\mathcal{W}_t =\left\{\min(1,t-\tau),\ldots, t-1 \right \}$:
				\[\hat \mu_{n_k(t)}^k \! = \! \frac{1}{n_k(t)}\! \sum_{s\in \mathcal{W}_t} \!\! X_s \indicator\{ \varphi_s=k \}, \,\, n_k(t) = \! \sum_{s\in \mathcal{W}_t} \!\! \indicator {\{ \varphi_s=k \}}\]
				
				Pick arm
				$\displaystyle \varphi_t=\arg \max_{k \in \mathcal{K}} \, \hat \mu_{n_k(t)}^k + \sqrt{\eta \frac{ \max  \Big(\ln \Big(\frac{\tau}{K n_k(t)}\Big), 0\Big)}{n_k(t)}}$;
				
			}				
		}

		% wait for next sample and go to step~\algostep{1}
		
		%  \nocaptionofalgo
		\caption{\textit{SW-MOSS}}
		\label{algo:sw-moss}}
\end{algorithm} 
\DecMargin{.3em}

\begin{fact}[Maximal Hoeffding inequality\cite{WH:63}] \label{fact: azuma}
	Let $X_1,\ldots, X_n$ be a sequence of independent $1/2$ sub-Gaussian random variables. Define $d_i := X_i- \mu_i $, then for any $ \delta > 0$, 
	\begin{align*}
		&\prob \bigg( \exists m \in \until{n} : \sum_{i=1}^{m} d_i \geq \delta \bigg)\leq \exp \left( -{2  \delta^2}/{n} \right)\\
		\text{and } & \prob \bigg(\exists m \in \until{n} : \sum_{i=1}^{m} d_i \leq -\delta \bigg)\leq \exp \left( -{2  \delta^2}/{n} \right).
	\end{align*}		
\end{fact}

At time $t$, for each arm $ k\in \mathcal{K}$ define
\[M_{t}^k :=\frac{1}{n_k(t)} \sum_{s\in \mathcal{W}_t} \mu_s^k \indicator_{\{ \varphi_s=k \}} .\]
Now, we are ready to present concentration bounds for the sliding window empirical mean $\hat{\mu}_{n_k(t)}^k$. 
\begin{lemma}\label{lemma: suff_sample}
	For any arm $k\in \mathcal{K}$ and any time $t \in \mathcal{T}$, if $\eta > 1/2 $, for any $x > 0$ and $l\geq 1$, the probability of event $A :=\big \{ \hat \mu_{n_k(t)}^k +c_{n_k(t)} \leq M_{t}^k - x, n_k(t) \geq l \big \}$ is no greater than
	\begin{equation} \label{quantitty1}
		\frac{(2\eta)^{\frac{3}{2}}}{\ln(2 \eta)}\frac{K }{\tau x^2 } \exp \left(-{x^2 l}/{\eta}\right).
	\end{equation}
	The probability of event $B :=\big \{ \hat \mu_{n_k(t)}^k - c_{n_k(t)}  \geq M_{t}^k + x, n_k(t) \geq l \big \} $ is also upper bounded by~\eqref{quantitty1}.
\end{lemma}
\begin{proof}
	For any $t \in \mathcal{T}$, let $u_i^{kt}$ be the $i$-th time slot when arm $k$ is selected within $\mathcal{W}_t$ and let $d_i^{kt} = X_{u_i^{kt}}^k - \mu_{u_i^{kt}}^k$. Note  that
	\[\prob \left(A \right) \leq \prob \bigg(\exists m \in \left\{l,\ldots, \tau \right\}:  \frac{1}{m} \sum_{i=1}^{m} d_i^{kt} \leq -x - c_m \bigg),\]
	Let $a = \sqrt{2\eta}$ such that $a>1$. We now apply a peeling argument~\cite[Sec 2.2]{bubeck2010bandits} with geometric grid $ a^s l < m \leq a^{s+1} l$ over $\left\{l,\ldots, \tau \right\}$. Since $c_m $ is monotonically decreasing in $m$,	
	\begin{align*}
	& \prob \bigg(\exists m \in \{l, \ldots, \tau\}:  \frac{1}{m} \sum_{i=1}^{m} d_i^{kt} \leq -x - c_m \bigg)\\
	\leq &\sum_{s\geq 0} \prob\bigg(\exists m \in [a^s l, a^{s+1}l) : \sum_{i=1}^{m} d_i^{kt} \leq -a^s l \left(x + c_{a^{s+1}l} \right) \bigg).	
	\end{align*}
	According to Fact~\ref{fact: azuma}, the above summand is no greater than
	\begin{align*}
	&\sum_{s\geq 0} \prob\bigg(\exists m \in [1, a^{s+1} l) : \sum_{i=1}^{m} d_i^{kt} \leq -a^s l \left(x + c_{a^{s+1}l} \right) \bigg) \\
	\leq &\sum_{s\geq 0} \exp \left( -2 \frac{a^{2s} l^2 }{\floor{a^{s+1}l}} \left(x^2 + c_{a^{s+1}l}^2\right) \right) \\
	\leq & \sum_{s\geq 0} \exp \left( -2 a^{s-1}l x^2 - \frac{2\eta}{a^2}\ln\left(\frac{\tau}{K a^{s+1}l}\right) \right) \\
	= & \sum_{s\geq 1} \frac{K l a^{s}}{\tau} \exp \left( -2 a^{s-2} l x^2\right).
	\end{align*}
	Let $b = 2x^2 l/a^2$. It follows that
	\begin{align*}
	&\sum_{s\geq 1} \frac{K l a^s}{\tau} \exp \left( -b a^s  \right) \leq \frac{Kl}{\tau} \int_{0}^{+\infty} a^{y+1} \exp\big(- b a^y\big) dy \\
	= & \frac{K l a}{\tau \ln(a)} \int_{1}^{+\infty} \exp (- b z ) dz \quad  \left(\text{where we set } z= a^y \right) \\
	= & \frac{K l a e^{-b}}{\tau b \ln(a)},
	\end{align*}
	which concludes the bound for the probability of event $A$. By using upper tail bound, similar result exists for event $B$.
\end{proof}

We now leverage Lemma~\ref{lemma: suff_sample} to get an upper bound on the worst-case regret for SW-MOSS.

\begin{theorem}\label{theorem:SW-MOSS bound}
	For the nonstationary MAB problem with $K$ arms, time horizon $T$, variation budget $V_T>0$ and $\tau = \Big \lceil { K^\frac{1}{3} \left(T/V_T \right) ^{\frac{2}{3}}} \Big \rceil$,  the worst-case regret of SW-MOSS  satisfies
	\[\sup_{\mathcal{F}_T^\mathcal{K} \in \mathcal{E}(V_T,T,K)} R_T^\text{SW-MOSS} \in \mc O((KV_T)^{\frac{1}{3}} T^{\frac{2}{3}}).\] 
\end{theorem}
\begin{proof}
	The proof consists of the following five steps.
	
	\noindent \textbf{Step 1:} Recall that $v_i $ is the variation within $\mathcal{T}_i$. Here, we trivially assign $\mathcal{T}_0 = \emptyset$ and $v_0 = 0$. Then, for each $i \in \until{N}$, let
	\[\Delta_{i}^k := \mu_{\tau_i}^{\kappa_i} - \mu_{\tau_i}^k - 2 v_{i-1} - 2 v_i, \quad \forall k \in \mathcal{K}.\]
	Define the set of bad arms for SW-MOSS in $\mathcal{T}_i$ as
	\[\supscr{\mathcal{B}_i}{SW} := \setdef{ k \in \mathcal{K} }{\Delta_{i}^k \geq \epsilon},\]
	where we assign $\epsilon = 4 \sqrt{e \eta K/\tau }$.
	
	\medskip	
	\noindent \textbf{Step 2:} We decouple the regret in this step. For any $t\in \mathcal{T}_i$,
	since $\abs{\mu_t^k - \mu_{\tau_i}^k} \leq v_i$ for any $ k \in \mathcal{K}$, it satisfies that 
	\begin{align*}
	\mu_{t}^* - \mu_{t}^{\varphi_t} & \leq  \mu_{\tau_i}^{\kappa_i} - \mu_{t}^{\varphi_t}\\
	&\leq \mu_{\tau_i}^{\kappa_i} - \mu_{\tau_i}^{\varphi_t} + v_i\\
	&\leq \indicator \left\{\varphi_t \in \supscr{\mathcal{B}_i}{SW}\right\} (\Delta_{i}^{\varphi_t} - \epsilon)  + 2v_{i-1} + 3 v_i + \epsilon.
	\end{align*}
	Then we get the following inequalities,
	\begin{align}
		&\sum_{t\in \mathcal{T}} \mu_{t}^* - \mu_{t}^{\varphi_t} \nonumber\\
		\leq& \sum_{i=1}^N  \sum_{t\in \mathcal{T}_i}  \indicator \left\{\varphi_t \supscr{\in \mathcal{B}_i}{SW}\right\} (\Delta_{i}^{\varphi_t} - \epsilon)  + 2v_{i-1} + 3 v_i + \epsilon \nonumber\\
		\leq&  5\tau V_T  + T\epsilon + \sum_{i=1}^N \sum_{t\in \mathcal{T}_i} \indicator \left\{\varphi_t \in \supscr{\mathcal{B}_i}{SW}\right\} (\Delta_{i}^{\varphi_t} - \epsilon). \label{regret_sw}
	\end{align}	
	To continue, we take a decomposition inspired by the analysis of MOSS in~\cite{MOSS} below,
	\begin{align}
	&\sum_{t \in \mathcal{T}_i} \indicator \left\{\varphi_t \in \supscr{\mathcal{B}_i}{SW}\right\} \left(\Delta_{i}^{\varphi_t} -\epsilon \right) \nonumber\\
	\leq & \sum_{t \in \mathcal{T}_i} \indicator \bigg \{\varphi_t \in \supscr{\mathcal{B}_i}{SW}, g_{t}^{\kappa_i} > M_{t}^{\kappa_i} - \frac{\Delta_{i}^{\varphi_t}}{4} \bigg\} \Delta_{i}^{\varphi_t} \label{overestimate} \\
	+& \sum_{t \in \mathcal{T}_i} \indicator \bigg\{\varphi_t \in \supscr{\mathcal{B}_i}{SW}, g_{t}^{\kappa_i} \leq M_{t}^{\kappa_i} - \frac{\Delta_{i}^{\varphi_t}}{4}\bigg\} \left(\Delta_{i}^{\varphi_t} -\epsilon \right) \label{underestimate},
	\end{align}
	where summands~\eqref{overestimate} describes the regret when arm $\kappa_i$ is fairly estimated and summand~\eqref{underestimate} quantifies the regret incurred by underestimating arm $\kappa_i$.
	
	\medskip	
	\noindent \textbf{Step 3:}	
	In this step, we bound $\expt\left[\eqref{overestimate}\right]$. Since $g_{t}^{\varphi_t} \geq g_{t}^{\kappa_i}$,	
	\begin{align}
		\eqref{overestimate} \leq & \sum_{t \in \mathcal{T}_i} \indicator \bigg\{\varphi_t \in \supscr{\mathcal{B}_i}{SW}, g_{t}^{\varphi_t} > M_{t}^{\kappa_i} - \frac{\Delta_{i}^{\varphi_t}}{4}\bigg\} \Delta_{i}^{\varphi_t} \nonumber \\
		=& \sum_{k \in \supscr{\mathcal{B}_i}{SW}} \sum_{t \in \mathcal{T}_i} \indicator \bigg\{\varphi_t = k, g_{t}^{k} > M_{t}^{\kappa_i} - \frac{\Delta_{i}^k}{4}\bigg\} \Delta_{i}^k.	\label{overestimate2}
	\end{align}
	Notice that for any $t \in \mathcal{T}_{i-1} \cup \mathcal{T}_i$,
	\[\abs{\mu_t^k - \mu_{\tau_i}^k} \leq v_{i-1} + v_i , \quad \forall k \in \mathcal{K}.\]
	It indicates that an arm $ k \in \supscr{\mathcal{B}_i}{SW}$ is at least $\Delta_i^k$ worse in mean reward than arm $\kappa_i$ at any time slot $t \in \mathcal{T}_{i-1} \cup \mathcal{T}_i$. Since $\mathcal{W}_t \subset \mathcal{T}_{i-1} \union  \mathcal{T}_i$, for any $t \in \mathcal{T}_i$
	\[M_t^{\kappa_i} - M_t^k \geq \Delta_{i}^{k} \geq \epsilon, \quad \forall k \in \supscr{\mathcal{B}_i}{SW}.\]
	It follows that
	\begin{equation}\label{overestimate3}
		\eqref{overestimate2} \leq \sum_{k \in \supscr{\mathcal{B}_i}{SW}} \sum_{t \in \mathcal{T}_i} \indicator \bigg\{\varphi_t = k, g_{t}^{k} > M_{t}^{k} + \frac{3\Delta_{i}^k}{4}\bigg\} \Delta_{i}^k.
	\end{equation}
	Let $t_s^{ik}$ be the $s$-th time slot when arm $k$ is selected within $\mathcal{T}_i$. Then, for any $k \in \supscr{\mathcal{B}_i}{SW}$,
	\begin{align}
	& \sum_{t \in \mathcal{T}_i} \indicator{ \bigg \{\varphi_t=k, g_{t}^k > M_{t}^{k} + \frac{3 \Delta_{i}^k}{4} \bigg\}} \nonumber \\
	= & \sum_{s\geq 1} \indicator{ \bigg\{g_{t_s^{ik}}^k > M_{t_s^{ik}}^{k} + \frac{3 \Delta_{i}^k}{4} \bigg\}} \nonumber \\
	\leq & l_i^k + \sum_{s\geq l_i^k + 1} \indicator{ \bigg\{g_{t_s^{ik}}^k > M_{t_s^{ik}}^{k} + \frac{3 \Delta_{i}^k}{4} \bigg\}}, \label{overestimate_k}
	\end{align}
	where we set $l_i^k = \bigg\lceil{\eta \Big(\frac{4}{\Delta_i^k}\Big)^{2} \ln \left ( \frac{\tau}{\eta K} \Big(\frac{\Delta_i^k}{4}\Big)^{2} \right )} \bigg \rceil$.
	Since $\Delta_i^k \geq \epsilon$, for $ k \in \supscr{\mathcal{B}_i}{SW}$, we have 
 $$l_i^k \geq \Big \lceil{\eta \left({4}/{\Delta_i^k}\right)^{2} \ln \left ( \frac{\tau}{\eta K} \left({\epsilon}/{4}\right)^{2} \right )} \Big \rceil \geq \eta \left({4}/{\Delta_i^k}\right)^{2}, $$
 where the second inequality follows by substituting $\epsilon = 4 \sqrt{e \eta K/\tau }$. Additionally, since $t_1^{ik}, \ldots,t_{s-1}^{ik} \in \mathcal{W}_{t_s^{ik}}$, we get $n_k(t_s^{ik})\geq s-1$. Furthermore, since $c_m$ is monotonically decreasing with $m$, 
 %$c_{n_k(t_s^{ik})} \leq c_{l_i^k} \leq {\Delta_i^k}/{4}$, . Accordingly, we get
	\[c_{n_k(t_s^k)} \leq c_{l_i^k} \leq \sqrt{\frac{\eta}{l_i^k} \ln \left(\frac{\tau}{\eta K} \bigg(\frac{\Delta_i^k}{4}\bigg)^{2} \right)} \leq \frac{\Delta_i^k}{4},\]
 for $s\geq l_i^k + 1$. Therefore,
	\[\eqref{overestimate_k} \leq l_i^k + \sum_{s\geq l_i^k + 1} \indicator{ \bigg\{g_{t_s^{ik}}^k - 2c_{n_k(t_s^{ik})} > M_{t_s^{ik}}^{k} + \frac{\Delta_{i}^k}{4} \bigg\}}.\]
	By applying Lemma~\ref{lemma: suff_sample}, considering $n_k(t_s^{ik})\geq s-1$,
	\begin{align}
	&\sum_{s\geq l_i^k + 1} \prob{ \bigg\{g_{t_s^{ik}}^k - 2c_{n_k(t_s^{ik})} > M_{t_s^{ik}}^{k} + \frac{\Delta_{i}^k}{4} \bigg\}} \nonumber\\
	\leq &\sum_{s \geq l_i^k} \frac{(2\eta)^{\frac{3}{2}}}{\ln(2 \eta)}\frac{K }{\tau}\bigg(\frac{4}{\Delta_i^k}\bigg)^{2}  \exp \left(-\frac{s}{\eta}\bigg(\frac{\Delta_i^k}{4}\bigg)^{2}\right) \nonumber\\
	\leq & \int_{l_i^k-1}^{+\infty} \frac{(2\eta)^{\frac{3}{2}}}{\ln(2 \eta)}\frac{K }{\tau} \bigg(\frac{4}{\Delta_i^k}\bigg)^{2}  \exp \left(-\frac{y}{\eta}\bigg(\frac{\Delta_i^k}{4}\bigg)^{2}\right) \, dy \nonumber\\
	\leq & \frac{(2\eta)^{\frac{3}{2}}}{\ln(2 \eta)}\frac{\eta K }{\tau} \bigg(\frac{4}{\Delta_i^k}\bigg)^{4} \label{overestimate_k3}.
	\end{align}
	Let $h(x) = 16 \eta/x \ln \left ( {\tau x^2}/{16 \eta K} \right )$ which achieves maximum at $4e \sqrt{ \eta K/\tau }$. Combining~\eqref{overestimate_k3},~\eqref{overestimate_k},~\eqref{overestimate3}, and~\eqref{overestimate2}, we obtain
	\begin{align*}
		\expt[\eqref{overestimate}] \leq &\sum_{k \in \mathcal{B}_i} \frac{(2\eta)^{\frac{3}{2}}}{\ln(2 \eta)}\frac{\eta K }{\tau} \frac{256}{\left(\Delta_i^k\right)^3} + l_i^k \Delta_i^k \\
		\leq & \sum_{k \in \mathcal{B}_i} \frac{(2\eta)^{\frac{3}{2}}}{\ln(2 \eta)}\frac{\eta K }{\tau} \frac{256}{\left(\Delta_i^k\right)^3} + h(\Delta_i^k) + \Delta_i^k   \\
		\leq & \sum_{k \in \mathcal{B}_i} \frac{(2\eta)^{\frac{3}{2}}}{\ln(2 \eta)}\frac{\eta K }{\tau} \frac{256}{\epsilon^3} + h \left(4 e\sqrt{ \eta K/\tau }\right) + b\\
		\leq & \left(\frac{2.6 \eta }{\ln(2 \eta)} + 3 \sqrt{\eta}\right) \sqrt{K\tau} + K b.
	\end{align*}
	
	\medskip
	\noindent \textbf{Step 4:}	
	In this step, we bound $\expt[\eqref{underestimate}]$. When event $\left\{\varphi_t \in \supscr{\mathcal{B}_i}{SW}, g_{t}^{\kappa_i} \leq M_{t}^{\kappa_i} - {\Delta_{i}^{\varphi_t}}/{4} \right\}$ happens, we know
	\[\Delta_i^{\varphi_t} \leq  4 M_{t}^{\kappa_i} - 4 g_{t}^{\kappa_i} \text{ and }  g_{t}^{\kappa_i} \leq M_{t}^{\kappa_i} - \frac{\epsilon}{4}. \]	
	Thus, we have
	\begin{align*}
	&\indicator \bigg\{\varphi_t \in \supscr{\mathcal{B}_i}{SW}, g_{t}^{\kappa_i} \leq M_{t}^{\kappa_i} - \frac{\Delta_{i}^{\varphi_t}}{4}\bigg\} \left(\Delta_{i}^{\varphi_t} -\epsilon \right) \\
	\leq &\indicator {\left\{ g_{t}^{\kappa_i} \leq M_{t}^{\kappa_i} - \frac{\epsilon}{4} \right \}} \times  \big(4 M_{t}^{\kappa_i} - 4 g_{t}^{\kappa_i} - \epsilon \big): = Y
	\end{align*}
	Since $Y$ is a nonnegative random variable, its expectation can be computed involving only its cumulative density function:
	\begin{align*}
	\expt  \left[Y\right] & = \int_{0}^{+\infty} \prob  \left( Y>x \right) dx \\
	& \leq \int_{0}^{+\infty} \prob \Big( 4 M_{t}^{\kappa_i} - 4 g_{t}^{\kappa_i} - \epsilon \geq x \Big) dx \\
	&= \int_{\epsilon}^{+\infty} \prob \Big( 4 M_{t}^{\kappa_i} - 4 g_{t}^{\kappa_i} > x \Big) dx\\
	& \leq \int_{\epsilon}^{+\infty} \frac{16 (2\eta)^{\frac{3}{2}}}{\ln(2 \eta)}\frac{K }{\tau x^2 } dx = \frac{16 (2\eta)^{\frac{3}{2}}}{\ln(2 \eta)}\frac{K }{\tau \epsilon}.
	\end{align*}
	Hence, $\expt[\eqref{underestimate} ] \leq  {16 (2\eta)^{\frac{3}{2}}}K\abs{\mathcal{T}_i}/\left(\ln(2 \eta) \tau \epsilon \right).$
	
	\medskip
	\noindent \textbf{Step 5:}	With bounds on $\expt\left[\eqref{overestimate}\right]$ and $\expt[\eqref{underestimate}]$ from previous steps,
	\begin{align*}
		\expt[\eqref{regret_sw}] \leq & 5\tau V_T + T\epsilon + N \left(\frac{2.6 \eta }{\ln(2 \eta)} + 3 \sqrt{\eta}\right) \sqrt{K\tau} \\
		&+NKb + \frac{16 (2\eta)^{\frac{3}{2}}}{\ln(2 \eta)}\frac{K T}{\tau \epsilon} \leq C(KV_T)^{\frac{1}{3}} T^{\frac{2}{3}}
	\end{align*}
	for some constant $C$, which concludes the proof.
\end{proof}

We have shown that SW-MOSS also enjoys order optimal worst-case regret. One drawback of the sliding window method is that all sampling history within the observation window needs to be stored. Since window size is selected to be $\tau =  \big \lceil { K^\frac{1}{3} ({T}/{V_T}  )^{\frac{2}{3}}} \big \rceil$, large memory is needed for large horizon length $T$. The next policy resolves this problem.

\subsection{Discounted UCB Algorithm}

The discount factor is widely used in estimators to forget old information and put more attention on the recent information. In~\cite{AG-EM:08}, such an estimation is used together with UCB$1$ to solve the piecewise stationary MAB problem, and the policy designed is called Discounted UCB (D-UCB). Here, we tune D-UCB to work in the nonsationary environment with variation budget $V_T$. Specifically, the mean estimator used is discounted empirical average given by
\begin{align*}
\hat \mu_{\gamma, t}^k &= \frac{1}{n_{\gamma,t}^k} \sum_{s= 1}^{t-1} \gamma^{t-s} \indicator \{ \varphi_s=k \} X_s,\\
n_{\gamma,t}^k&=\sum_{s = 1}^{t-1} \gamma^{t-s} \indicator \{ \varphi_s=k \},
\end{align*}
where $\gamma = 1- { K^{-\frac{1}{3}} ({T}/{V_T}  )^{-\frac{2}{3}}}$ is the discount factor.
Besides, the UCB is designed as $g_t^k = \hat \mu_t^k + 2 c_t^k$, where $c_{\gamma,t}^k = \sqrt{ \xi \ln (\tau) / n_{\gamma, t}^k}$ for some constant $\xi > 1/2$. The pseudo code for D-UCB is reproduced in Algorithm~\ref{algo:d-moss}. It can be noticed that the memory size is only related to the number of arms, so D-UCB requires small memory.

\IncMargin{.3em}
% \restylealgo{boxed}
% \linesnumbered 
\begin{algorithm}[t]
{\footnotesize
	\SetKwInOut{Input}{  Input}
	\SetKwInOut{Set}{  Set}
	\SetKwInOut{Title}{Algorithm}
	\SetKwInOut{Require}{Require}
	\SetKwInOut{Output}{Output}
	
	{
		%		\it For abruptly-changing environment} \\   
		
		\Input{$V_T \in \real_{>0}$, $T\in \natural$ and $\xi > \frac{1}{2}$}
		
		\smallskip
		
		\Set{$\gamma = 1- { K^{-\frac{1}{3}} ({T}/{V_T}  )^{-\frac{2}{3}}}$}

	}
	
	\medskip

	\Output{sequence of arm selection}
	
	\medskip
	
	\nl  \For{$t \in \until{K}$ \smallskip}
	{Pick arm $\varphi_t=t$ and set  $n^t \leftarrow \gamma^{K-t}$ and $\hat \mu^t \leftarrow X_t^t $;
	}

	\smallskip
	
	\nl \While{$t \leq T$}{
		\smallskip				

		Pick arm 
			$\displaystyle
			\varphi_t=\arg \max_{k \in \mathcal{K}} \hat \mu^k + 2 \sqrt{\frac{ \xi \ln (\tau) }{ n^k}}
			$;
			
			\smallskip		
			For each arm $k \in \mathcal{K}$, set $ n^k \leftarrow \gamma n^k $;
			
			\smallskip
			\smallskip
			Set
			$ n^{\varphi_t} \leftarrow n^{\varphi_t}+1 \: \& \: \hat \mu^{\varphi_t} \leftarrow \hat  \mu^{\varphi_t}+ \frac{1}{n^{\varphi_t}} (X_t^{\varphi_t}-\bar X^{\varphi_t});$
		
		\smallskip					
	}

	% wait for next sample and go to step~\algostep{1}
	
	%  \nocaptionofalgo
	\caption{\textit{D-UCB}}
	\label{algo:d-moss}}
\end{algorithm} 
\DecMargin{.3em}

To proceed the analysis, we review the concentration inequality for discounted empirical average, which is an extension of Chernoff-Hoeffding bound. Let 
\[M_{\gamma, t}^k := \frac{1}{n_{\gamma,t}^k} \sum_{s= 1}^{t-1} \gamma^{t-s} \indicator \{ \varphi_s=k \} \mu_s^k.\]
Then, the following fact is a corollary of~\cite[Theorem 18]{AG-EM:08}.

\begin{fact}[A Hoeffding-type inequality for discounted empirical average with a random number of summands] \label{discounted_ineq} 
	For any $t \in \mathcal{T}$ and for any $k \in \mathcal{K}$, the probability of event $A = \left\{ {\hat \mu_{\gamma, t}^k - M_{\gamma,t}^k} \geq \delta / \sqrt{n_{\gamma,t}^k} \right\} $ is no greater than
	\begin{equation}\label{quantitty2}
		\ceil{\log_{1+\lambda} ( \tau)} \exp \left(-2\delta^2 \big(1-{\lambda^2}/{16} \big)\right)
	\end{equation}
	for any $\delta > 0$ and $\lambda >0 $. The probability of event $B = \left\{ \hat \mu_{\gamma, t}^k - M_{\gamma, t}^k \leq -\delta/\sqrt{n_{\gamma, t}^k} \right\}$ is also upper bounded by~\eqref{quantitty2}.
\end{fact}

\begin{theorem}\label{theorem:D-UCB bound}
	For the nonstationary MAB problem with $K$ arms, time horizon $T$, variation budget $V_T>0$, and $\gamma = 1- { K^{-\frac{1}{3}} ({T}/{V_T}  )^{-\frac{2}{3}}}$,  if $\xi>1/2$, the worst case regret of D-UCB satisfies
	\[\sup_{\mathcal{F}_T^\mathcal{K} \in \mathcal{E}(V_T,T,K)} \supscr{ R_T }{D-UCB}\leq C \ln (T) (KV_T)^{\frac{1}{3}} T^{\frac{2}{3}} .\]
\end{theorem}

\begin{proof}
We establish the theorem in four steps. 

\noindent \textbf{Step 1:}
In this step, we analyze $\big|{\mu_{\gamma, t}^k - M_{\gamma,t}^k}\big|$ at some time slot $t \in \mathcal{T}_i$. Let $\tau' = {\log_{\gamma}\big((1-\gamma) \xi \ln (\tau)/b^2  \big)}$ and take $t-\tau'$ as a dividing point, then we obtain
\begin{align}
\abs{\mu_{\tau_i}^k - M_{\gamma,t}^k} \leq& \frac{1}{n_{\gamma, t}^k } \sum_{s = 1}^{t-1} \gamma^{t-s} \indicator \{ \varphi_s=k \}  \abs{ \mu_{\tau_i}^k - \mu_s^k } \nonumber\\
\leq& \frac{1}{n_{\gamma, t}^k } \sum_{s \leq t-\tau'} \gamma^{t-s} \indicator \{ \varphi_s=k \} \abs{ \mu_{\tau_i}^k - \mu_s^k } \label{bias: dis} \\
+& \frac{1}{n_{\gamma, t}^k } \sum_{s\geq t-\tau'}^{t-1} \gamma^{t-s} \indicator \{ \varphi_s=k \} \abs{ \mu_{\tau_i}^k - \mu_s^k } \label{bias: dis2}.
\end{align}
Since $\mu_t^k \in [a,a+b]$ for all $t\in \mathcal{T}$, we have $\eqref{bias: dis}\leq b$. Also,
\[{\eqref{bias: dis}} \leq \frac{1}{n_{\gamma,t}^k} \sum_{s \leq t-\tau'} b \gamma^{t-s}  \leq \frac{b \gamma^{\tau'}}{(1-\gamma) n_{\gamma,t}^k} = \frac{ \xi \ln (\tau) }{b n_{\gamma,t}^k}.\]
Accordingly, we get
\[{\eqref{bias: dis}} \leq \min \left(b,\frac{ \xi \ln (\tau) }{b n_{\gamma,t}^k} \right)  \leq \sqrt{\frac{ \xi \ln (\tau) }{ n_{\gamma,t}^k}}.\]
Furthermore, for any $t \in \mathcal{T}_i$,
\[\eqref{bias: dis2} \leq \max_{s\in [t-\tau',t-1]} \abs{ \mu_{\tau_i}^k - \mu_s^k } \leq \sum_{j= i - n' }^i v_j,\]
where $n' = \lceil {\tau'/\tau} \rceil$ and $v_j$ is the variation within $\mathcal{T}_j$. So we conclude that for any $t \in \mathcal{T}_i$,
\begin{equation}\label{bias}
	\abs{\mu_{\kappa_i}^k - M_{\gamma,t}^k} \leq c_{\gamma,t}^k + \sum_{j= i - n' }^i v_j, \quad \forall k \in \mathcal{K}.
\end{equation}

\medskip
\noindent \textbf{Step 2:} Within partition $\mathcal{T}_i$, let 
\[\hat{\Delta}_i^k=  \mu_{\tau_i}^{\kappa_i} - \mu_{\tau_i}^k - 2\sum_{j= i - n' }^i v_j,\]
and define a subset of bad arms as
\[\supscr{\mathcal{B}_i}{D} = \bigg\{ k \in \mathcal{K} \: | \: \hat{\Delta}_i^k \geq \epsilon' \bigg\},\] 
where we select $\epsilon' = 4\sqrt{\xi \gamma^{1-\tau} K \ln(\tau)/\tau}$. Since $\abs{\mu_t^k - \mu_{\tau_i}^k} \leq v_i$ for any $t\in \mathcal{T}_i$ and for any $ k \in \mathcal{K}$
\begin{align}
&\sum_{t\in \mathcal{T}} \mu_{t}^* - \mu_{t}^{\varphi_t} \leq \sum_{i=1}^N \sum_{t\in \mathcal{T}_i}\mu_{\tau_i}^{\kappa_i} - \mu_{\tau_i}^{\varphi_t} + v_i \nonumber\\
\leq& \tau V_T + \sum_{i=1}^N  \sum_{t\in \mathcal{T}_i}  \bigg[\indicator \left\{\varphi_t \in \supscr{\mathcal{B}_i}{D} \right\} \hat \Delta_{i}^{\varphi_t}  + 2\sum_{j= i - n' }^i v_j + \epsilon' \bigg]\nonumber\\
\leq& (2n'+3) \tau V_T + N\epsilon ' \tau \! + \! \sum_{i=1}^N \sum_{k \in \supscr{\mathcal{B}_i}{D}} \! \! \hat \Delta_{i}^{k} \sum_{t\in \mathcal{T}_i} \indicator \left\{\varphi_t = k \right\}. \label{regret: decouple}
\end{align}

\medskip
\noindent \textbf{Step 3:}
In this step, we bound $\expt \big[ \hat \Delta_{i}^{k} \sum_{t\in \mathcal{T}_i} \indicator \left\{\varphi_t = k \right\} \big]$ for an arm $k\in \supscr{\mathcal{B}_i}{D}$. Let $t_i^k(l)$ be the $l$-th time slot arm $k$ is selected within $\mathcal{T}_i$. From arm selection policy, we get $g_{t}^{\varphi_t} \geq g_{t}^{\kappa_i}$, which result in
\begin{equation}\label{bound: seleted_num}
	\sum_{t\in \mathcal{T}_i} \indicator \left\{\varphi_t = k \right\} \leq l_i^k + \sum_{ t \in \mathcal{T}_i } \indicator{ \Big\{g_{t}^k \geq g_{t}^{\kappa_i}, t> t_i^k(l_i^k) \Big\}},
\end{equation}
where we pick $l_i^k = \ceil{{16 \xi \gamma^{1-\tau}\ln (\tau) / {\big(\hat\Delta_i^k \big)^2}}}$. Note that $g_t^k \geq g_t^{\kappa_i}$ is true means at least one of the followings holds,
\begin{align}
\hat \mu_{\gamma,t}^k &\geq  M_{\gamma,t}^k + c_{\gamma,t}^k, \label{h1}\\
\hat \mu_{\gamma,t}^{\kappa_i} &\leq  M_{\gamma,t}^{\kappa_i} - c_{\gamma,t}^{\kappa_i}, \label{h2}\\
M_{\gamma,t}^{\kappa_i} + c_{\gamma,t}^{\kappa_i} &<  M_{\gamma,t}^k + 3 c_{\gamma,t}^k.\label{h3}
\end{align}
For any $t \in \mathcal{T}_i$, since every sample before $t$ within $\mathcal{T}_i$ has a weight greater than $\gamma^{\tau-1}$, if $t > t_i^k(l_i^k)$,
\begin{align*}
	c_{\gamma, t}^k = \sqrt{\frac{ \xi \ln (\tau) }{ n_{\gamma,t}^k}} \leq \sqrt{\frac{ \xi \ln (\tau) }{ \gamma^{\tau-1} l_i^k }} \leq \frac{\hat \Delta_i^k}{4}.
\end{align*}
Combining it with~\eqref{bias} yields
\begin{align*}
M_{\gamma,t}^{\kappa_i} - M_{\gamma,t}^k & \geq \mu_{\tau_i}^{\kappa_i} - \mu_{\tau_i}^k - c_{\gamma,t}^{\kappa_i} - c_{\gamma,t}^{k} - 2\sum_{j= i - n' }^i v_j \\
&\geq \hat{\Delta}_i^k - c_{\gamma,t}^{\kappa_i} - c_{\gamma,t}^k \geq 3c_{\gamma,t}^k - c_{\gamma,t}^{\kappa_i},
\end{align*}
which indicates~\eqref{h3} is false. As $\xi > 1/2$, we select $\lambda = 4 \sqrt{1-1/(2\xi)}$ and apply Fact~\ref{discounted_ineq} to get
\[\prob(\text{\eqref{h1} is true}) \leq \ceil{\log_{1+\lambda} (\tau)} \tau^{-2\xi (1-{\lambda^2}/{16} )} \leq \frac{\ceil{\log_{1+\lambda} (\tau)}}{\tau}.\]
The probability of~\eqref{h2} to be true shares the same bound. Then, it follows from~\eqref{bound: seleted_num} that $\expt \big[ \hat \Delta_{i}^{k} \sum_{t\in \mathcal{T}_i} \indicator \left\{\varphi_t = k \right\} \big]$ is upper bounded by
\begin{align}
& \hat \Delta_{i}^{k} l_i^k + \hat \Delta_{i}^{k} \sum_{t\in \mathcal{T}_i}  \prob \left( \text{\eqref{h1} or~\eqref{h2} is true}\right) \nonumber\\
\leq& \frac{16 \xi \gamma^{1-\tau} \ln(\tau)}{\hat\Delta_i^k } +\hat\Delta_{i}^{k}  + 2 \hat \Delta_{i}^{k} \ceil{\log_{1+\lambda} \left(\tau\right)} \nonumber\\
\leq& \frac{16 \xi \gamma^{1-\tau} \ln(\tau)}{\epsilon'} + b + 2b \ceil{\log_{1+\lambda} \left(\tau\right)}, \label{regret: estimation}
\end{align}
where we use $\epsilon' \leq \hat\Delta_i^k \leq b$ in the last step. 

\noindent \textbf{Step 4:}
From~\eqref{regret: decouple} and~\eqref{regret: estimation}, and plugging in the value of $\epsilon'$, an easy computation results in
\begin{align*}
	\supscr{R_T}{D-UCB} \leq &(2n'+3) \tau V_T + 8N \sqrt{\xi \gamma^{1-\tau} K \tau \ln(\tau)} \\
	&+ 2Nb+ 2 N b\log_{1+\lambda} \left(\tau\right),
\end{align*}
where the dominating term is $(2n'+3) \tau V_T$. Considering 
\[\tau' = \frac{\ln\big((1-\gamma) \xi \ln (\tau)/b^2  \big)}{\ln {\gamma}} \leq \frac{-\ln \big((1-\gamma) \xi \ln (\tau)/b^2  \big)}{1-\gamma},\] we get $n'\leq C'\ln(T)$ for some constant $C'$. Hence there exists some absolute constant $C$ such that
\[\supscr{R_T}{D-UCB} \leq C \ln (T) (KV_T)^{\frac{1}{3}} T^{\frac{2}{3}} .\]
\end{proof}

Although discount factor method requires less memory, there exists an extra factor $\ln(T) $ in the upper bound on the worst-case regret for D-UCB comparing with the minimax regret. This is due to the fact that discount factor method does not entirely cut off outdated sampling history like periodic resetting or sliding window techniques.

\section{UCB Policies for Heavy-tailed Nonstationary Stochastic MAB Problems} \label{sec: heavy-tailed}

In this section, we propose and analyze UCB algorithms for non-stationary stochastic MAB problem with heavy-tailed rewards defined in Assumption~\ref{ass: heavy-tailed}.  
%
% show that the performance guarantees from previous section can be maintained with a weaker assumption: the rewards are allowed to have heavy-tailed distributions instead of only sub-Gaussian distributions. Formally, it is presented below.
%
% \begin{assumption}\label{ass: heavy-tailed}
% 	Let $X$ be a random reward drawn from any arm $k \in \until{K}$. There exists a constant $u \in \real_+$ such that $\expt \left[X^2\right] \leq 1$.
% \end{assumption}
% We will adopt a robust mean estimator  heavy-tailed bandits, i.e., using  to construct upper confidence bound. 
We first recall a minimax policy for the stationary heavy-tailed MAB problem called Robust MOSS~\cite{wei2020minimax}. We then extend it to nonstationary setting and design  resetting robust MOSS algorithm and sliding-window robust MOSS algorithm.

\subsection{Background on Robust MOSS algorithm for the stationary heavy-tailed MAB problem}
Robust MOSS algorithm handles stationary heavy-tailed MAB problems in which the rewards have finite moments of order $1+\epsilon$, for $\epsilon \in (0,1]$. For simplicity, as stated in Assumption~\ref{ass: heavy-tailed}, we restrict our discussion to $\epsilon =1$. 

Robust MOSS uses the saturated empirical mean instead of the empirical mean. Let $n_k(t)$ be the number of times that arm $k$ has been selected until time $t-1$. Pick $a >1$ and let $h(m) = a^{\floor{\log_a \left(m\right)}+1}$. Let the saturation limit at time $t$ be defined by
\[ B_{n_k(t)} := \sqrt{\frac{h(n_k(t))}{\ln_+ \left(\frac{T}{K h(n_k(t))}\right)}},\]
where $\ln_+(x) := \max (\ln x, 1)$. 
Then, the saturated empirical mean estimator is defined by
\begin{equation}\label{def: sat_mean}
\bar \mu_{n_k(t)} := \frac{1}{n_k(t)} \sum_{s=1}^{t-1} \indicator\{\varphi_s= k \}\sat (X_s,B_{n_k(t)}),
\end{equation}
where $\sat (X_s,B_m) := \sign(X_s)\min \big\{\abs{X_s}, B_m \big\}.$ The Robust MOSS algorithm initializes by selecting each arm once and subsequently, at each time $t$, selects the arm that maximizes the following upper confidence bound
\[
g^k_{n_k(t)} = \bar \mu^k_{n_k(t)} + (1+\zeta)c_{n_k(t)}, 
\]
where $c_{n_k(t)} = \sqrt{{\ln_+ \big(\frac{T}{K n_k(t)}\big)}/{n_k(t)}}$, $\zeta$ is an positive constant such that $\psi(2\zeta/a) \geq 2a/\zeta$ and $\psi(x) =  (1+1/x)\ln(1+x)-1 $. Note that for $x\in (0,\infty)$, function $\psi(x)$ is monotonically increasing in $x$.

\subsection{Resetting robust MOSS for the non-stationary heavy-tailed MAB problem}

Similarly to R-MOSS, Resetting Robust MOSS (R-RMOSS) restarts Robust MOSS after every $\tau$ time slots. For a stationary heavy-tailed MAB problem, it has been shown in~\cite{wei2020minimax} that the worst-case regret of Robust MOSS belongs to $\mc O(\sqrt{KT})$. This result along with an analysis similar to the analysis for R-MOSS in Theorem~\ref{theorem:R-MOSS bound} yield the 
following theorem for R-RMOSS. For brevity, we skip the proof. 
% It shows that the order optimal worst-case regret can be maintained.
\begin{theorem}\label{theorem:R-RobustMOSS bound}
	For the nonstationary heavy-tailed MAB problem with $K$ arms, horizon $T$, variation budget $V_T>0$ and $\tau = \Big \lceil { K^\frac{1}{3} \left(T/V_T \right) ^{\frac{2}{3}}} \Big \rceil$, if $\psi(2\zeta/a) \geq 2a/\zeta$, the worst-case regret of R-RMOSS satisfies
	\[\sup_{\mathcal{F}_T^\mathcal{K} \in \mathcal{E}(V_T,T,K)} \supscr{R_T}{R-RMOSS} \in  \mc O((KV_T)^{\frac{1}{3}} T^{\frac{2}{3}}).\] 
\end{theorem}

\subsection{Sliding-window robust MOSS for the non-stationary heavy-tailed MAB problem}
In Sliding-Window Robust MOSS (SW-RMOSS), $n_k(t)$ and $\bar{\mu}_{n_k(t)}$ are computed from the sampling history within $\mathcal{W}_t$, and $c_{n_k(t)} = \sqrt{{\ln_+ \big(\frac{\tau}{K n_k(t)}\big)}/{n_k(t)}}$. To analyze SW-RMOSS, we want to establish a similar property as Lemma~\ref{lemma: suff_sample} to bound the probability about an arm being under or over estimated. Toward this end, we need the following properties for truncated random variable.

\begin{lemma}\label{bias: d}
	Let $X$ be a random variable with expected value $\mu$ and $\expt[X^2] \leq 1$. Let $d: = \sat(X,B)-\expt [\sat(X,B)]$. Then for any $B> 0$, it satisfies (i) $\abs{d}\leq 2 B$ (ii) $\expt[d^2] \leq 1$ (iii) $\abs{\expt [\sat(X,B)] - \mu} \leq 1/B$.
\end{lemma}

\begin{proof}
	Property (i) follows immediately from definition of $d$  and property (ii) follows from
	\[ \expt[d^2] \leq \expt\big[\sat^2(X,B)\big] \leq  \expt\big[X^2\big].\]
	To see property (iii), since \[\mu = \expt \big[ X \big(\indicator { \left\{ \abs{X} \leq B\right\} } + \indicator{ \left\{ \abs{X} > B\right\} }\big)\big],\] one have	
	\begin{align*}
	\abs{\expt [\sat(X,B)] - \mu} &\leq \expt \left[ \left(\abs{X}-B\right) \indicator{ \left\{ \abs{X} > B\right\} } \right] \\
	&\leq \expt \left[ \abs{X} \indicator{ \left\{ \abs{X} > B\right\} } \right] \leq \expt \left[{X^2}/{B} \right].
	\end{align*}
\end{proof}
\noindent
Moreover, we will also use a maximal Bennett type inequality as shown in the following. 
\begin{lemma}[Maximal Bennett's inequality~{\cite{fan2012hoeffding}}] \label{max_inq_b}
	Let $\seqdef{X_i}{i\in \until{n}}$ be a sequence of bounded random variables with support $[-B,B]$, where $B\geq 0$. Suppose that $\expt[X_i |X_{1},\ldots,X_{i-1}] = \mu_i$ and $\Var[X_i|X_{1},\ldots,X_{i-1}] \leq v$. Let $S_m = \sum_{i=1}^{m} (X_i -\mu_i) $ for any $m\in \until{n}$. Then, for any $\delta \geq 0$
	\begin{align*}
	&\prob\left( \exists {m \in \until{n}}:  S_m \geq \delta \right) \leq \exp \left( -\frac{\delta}{B}\psi \left (\frac{B\delta}{n v} \right) \right), \\
	&\prob\left(\exists {m \in \until{n}}: S_m \leq -\delta\right) \leq \exp \left(-\frac{\delta}{B}\psi \left (\frac{B\delta}{n v} \right)\right).
	\end{align*} 
\end{lemma}

Now, we are ready to establish a concentration property for saturated sliding window empirical mean.
\begin{lemma}\label{lemma: suff_sample_heavy}
	For any arm $k\in \until{K}$ and any $t \in \left\{K+1,\ldots,T\right\}$, if $\psi(2\zeta/a) \geq 2a/\zeta$, the probability of either event $A=\big \{ g^k_{t}  \leq M_{t}^k - x, n_k(t) \geq l \big \}$ or event $B=\big \{ g^k_{t} - 2 c_{n_k(t)}  \geq M_{t}^k + x, n_k(t) \geq l \big \} $, for any $x > 0$ and any $l\geq 1$,  is no greater than
	\[\frac{2a}{\beta^2  \ln(a) } \frac{K}{\tau x^2} (\beta x\sqrt{h(l)/a}+1) \exp \left(- \beta x \sqrt{h(l)/a}\right),\]
	where $\beta = \psi \left ( 2\zeta /a \right)/ (2a) $.
\end{lemma}
\begin{proof}
    Recall that $u_i^{kt}$ is the $i$-th time slot when arm $k$ is selected within $\mathcal{W}_t$.	Since $c_m$ is a monotonically decreasing in $m$, $1/B_m = c_{h(m)} \leq c_m $ due to $h(m)\geq m$. Then, it follows from property (iii) in Lemma~\ref{bias: d} that
	\begin{align}
	\prob(A)\!	&\leq  \prob \bigg(\! \exists m \!\in \!\{l,\ldots, \tau\}  \!:\!  \bar{\mu}^k_m \leq  \sum_{i=1}^{m} \! \frac{ \mu_{u_i^{kt}}^k}{m}\!\! - (1+\zeta) c_m \!-x \! \bigg) \nonumber	\\
	&\leq \prob \bigg( \! \exists m \!\in \! \{l,\ldots, \tau\}\!:\!  \sum_{i=1}^{m} \! \frac{\bar{d}_{im}^{kt}}{m}\! \leq \! \frac{1}{B_m} \! - (1+\zeta) c_m \!-x \! \bigg) \nonumber	\\
	&\leq \prob \bigg(\! \exists m \!\in \! \{l,\ldots, \tau\} \!:\!  \frac{1}{m} \sum_{i=1}^{m} \bar{d}_{im}^{kt} \leq -x - \zeta c_m \bigg), \, \label{prob:A}
	\end{align}
	where $\bar{d}_{im}^{kt} = \sat \big(X_{u_i^{kt}}^k,B_m\big)-\expt \big[\sat \big(X_{u_i^{kt}}^k,B_m \big)\big]$. Recall we select $a>1$. Again, we apply a peeling argument with geometric grid $ a^s \leq m < a^{s+1}$ over time interval $\{l,\ldots, \tau\}$. Let $s_0 = \floor{\log_a (l)}$. Since $c_m$ is monotonically decreasing with $m$,
	\[\eqref{prob:A} \leq \!\! \sum_{s\geq s_0} \!\prob\Bigg(\!\exists m \in [a^s, a^{s+1}) \!:\! \sum_{i=1}^{m} \bar{d}_{im}^{kt} \leq \! -a^s \left(x + \zeta c_{a^{s+1}} \right) \!\! \bigg).\]
	For all $m \in [a^s, a^{s+1})$, since $B_m = B_{a^s}$, from Lemma~\ref{bias: d} we know  $\abs{\bar{d}_{im}^{kt}}\leq 2 B_{a^s}$ and $\var \left[ \bar{d}_{im}^{kt}\right] \leq 1$. Continuing from previous step, we apply Lemma~\ref{max_inq_b} to get	
	\begin{align}
	\eqref{prob:A}\leq &\sum_{s\geq s_0} \exp \left( -\frac{a^{s} \left(x+ \zeta c_{a^{s+1}}\right)}{2 B_{a^{s}}}\psi \left ( \frac{2 B_{a^{s}}}{a} \left(x + \zeta c_{a^{s+1}}\right) \right) \right)\nonumber\\
	&\left(\text{since } \psi(x) \text{ is monotonically increasing}\right)\nonumber\\
	\leq & \sum_{s\geq s_0} \exp \left( -\frac{a^{s} \left(x+ \zeta c_{a^{s+1}}\right)}{2 B_{a^{s}}}\psi \left (\frac{ 2 \zeta}{a} B_{a^{s}} c_{a^{s+1}} \right) \right) \nonumber\\
	&\text{(substituting $c_{a^{s+1}}$, $B_{a^{s}}$ and using $h(a^s)=a^{s+1}$)} \nonumber \\
	=& \sum_{s\geq s_0 + 1} \exp \left( -a^{s}\left(\frac{ x}{B_{a^{s-1}}} + \zeta c_{a^s}^2\right) \frac{\psi \left ( 2\zeta/a \right)}{2a} \right)  \nonumber \\
	&\left(\text{since } \zeta\psi(2\zeta/a)\geq 2a \right)\nonumber\\
	\leq & \frac{K}{\tau} \sum_{s\geq s_0 + 1} a^s \exp \left( -a^{s} \frac{ x}{B_{a^{s-1}}} \frac{\psi \left ( 2\zeta /a \right)}{2a} \right). \label{sum:1}
	\end{align}
	Let $b = {x\psi \left ( 2\zeta /a \right)}/ (2a)$. Since $\ln_+(x) \geq 1$ for all $x>0$,
	\begin{align*}
	\eqref{sum:1}\leq & \frac{K}{\tau} \sum_{s\geq s_0 + 1} a^s \exp \left( -b \sqrt{a^s} \right) \\
	\leq & \frac{K}{\tau}\int_{s_0+1}^{+\infty} a^y \exp\big(- b \sqrt{a^{y-1}}\big) dy  \\
	= &  \frac{K}{\tau} a\int_{s_0}^{+\infty} a^y \exp\big(-b \sqrt{a^{y}}\big) dy  \\
	= & \frac{K}{\tau} \frac{2a}{\ln(a)b^2}\int_{b \sqrt{a^{s_0}}}^{+\infty} z \exp\big(- z  \big) dz \, (\text{where } z=b \sqrt{a^y}) \\
	\leq & \frac{K}{\tau} \frac{2 a}{\ln(a)b^2}  (b\sqrt{a^{s_0}}+1) \exp(-b\sqrt{a^{s_0}}),
	\end{align*}
	which concludes the proof. 
\end{proof}
With Lemma~\ref{lemma: suff_sample_heavy}, the upper bound on the worst-case regret for SW-RMOSS in the nonstationary heavy-tailed MAB problem can be analyzed similarly as Theorem~\ref{theorem:SW-MOSS bound}.
\begin{theorem}\label{theorem:SW-RobustMOSS bound}
	For the nonstationary heavy-tailed MAB problem with $K$ arms, time horizon $T$, variation budget $V_T>0$ and $\tau = \Big \lceil { K^\frac{1}{3} \left(T/V_T \right) ^{\frac{2}{3}}} \Big \rceil$, if $\psi(2\zeta/a) \geq 2a/\zeta$, the worst-case regret of SW-RMOSS satisfies
	\[\sup_{\mathcal{F}_T^\mathcal{K} \in \mathcal{E}(V_T,T,K)} \supscr{R_T}{SW-RMOSS} \leq C(KV_T)^{\frac{1}{3}} T^{\frac{2}{3}}.\] 
\end{theorem}
\begin{proof}[Sketch of the proof]
	The procedure is similar as the proof of Theorem~\ref{theorem:SW-MOSS bound}. The key difference is due to the nuance between the concentration properties on mean estimator. Neglecting the leading constants, the probability upper bound in Lemma~\ref{lemma: suff_sample} has a factor $\exp(-x^2l/\eta)$ comparing with $(\beta x\sqrt{h(l)/a}+1) \exp \left(- \beta x \sqrt{h(l)/a}\right)$ in Lemma~\ref{lemma: suff_sample_heavy}. Since both factors are no greater than $1$, by simply replacing $\eta$ with $(1+\zeta)^2$ and taking similar calculation in every step except inequality~\eqref{overestimate_k3}, comparable bounds that only differs in leading constants can be obtained. Applying Lemma~\ref{lemma: suff_sample_heavy}, we revise the computation of~\eqref{overestimate_k3} as the following,
	\begin{align}
	&\sum_{s\geq l_i^k + 1} \prob{ \bigg\{g_{t_s}^k - 2c_{n_k(t_s)} > M_{t_s}^{k} + \frac{\Delta_{i}^k}{4} \bigg\}} \nonumber\\
	\leq &\sum_{s \geq l_i^k} C' \left(\frac{\beta\Delta_i^k}{4}\sqrt{\frac{h(l)}{a}}+1\right) \exp \left(-\frac{\beta\Delta_i^k}{4}\sqrt{\frac{h(l)}{a}}\right) \nonumber\\
	\leq & \int_{l_i^k-1}^{+\infty} C' \left(\frac{\beta\Delta_i^k}{4}\sqrt{\frac{y}{a}}+1\right)  \exp \left(-\frac{\beta\Delta_i^k}{4}\sqrt{\frac{y}{a}}\right) \, dy \nonumber\\
	\leq & \frac{6 a}{\beta^2}\frac{2a}{\beta^2  \ln(a) } \frac{K}{\tau}\bigg(\frac{4}{\Delta_i^k}\bigg)^4. 
	\end{align}
	where $C'= {2aK}\big({4}/{\Delta_i^k}\big)^{2}/{\big(\beta^2  \ln(a) \tau\big) }$.The second inequality is due to the fact that $(x+1)\exp(-x)$ is monotonically decreasing in $x$ for $x\in[0,\infty)$ and $h(l)>l$. In the last inequality, we change the lower limits of the integration from $l_i^k-1$ to $0$ since $l_i^k\geq 1$ and plug in the value of $C'$. Comparing with~\eqref{overestimate_k3}, this upper bound only varies in constant multiplier. So is the worst-regret upper bound.
\end{proof}

\begin{remark}
	The benefit of discount factor method is that it is memory friendly. This advantage is lost if truncated empirical mean is used. As $n_k(t)$ could both increase and decrease with time, the truncated point could both grow and decline, so all sampling history needs to be recorded. It remains an open problem how to effectively using discount factor in a nonstationary heavy-tailed MAB problem.
\end{remark}

\section{Numerical Experiments} \label{sec: simulation}
We complement the theoretical results in previous section with two Monte-Carlo experiments. For the light-tailed setting, we compare R-MOSS, SW-MOSS and D-UCB in this paper with other state-of-art policies. For the heavy-tailed setting, we test the robustness of R-RMOSS and SW-RMOSS against both heavy-tailed rewards and nonstationarity. Each result in this section is derived by running designated policies $500$ times. And parameter selections for compared policies are strictly coherent with referred literature.

\subsection{Bernoulli Nonstationay Stochastic MAB Experiment}
To evaluated the performance of different policies, we consider two nonstationary environment as shown in Figs.~\ref{1rst} and~\ref{2nd}, which both have $3$ arms with nonstationary Bernoulli reward. The success probability sequence at each arm is a Brownian motion in environment $1$ and a sinusoidal function of time $t$ in environment $2$. And the variation budget $V_T$ is $8.09$ and $3$ respectively.

\begin{figure}[ht!]
	\centering
	\begin{subfigure}[b]{0.24\textwidth}
		\centering
		\includegraphics[width=\textwidth]{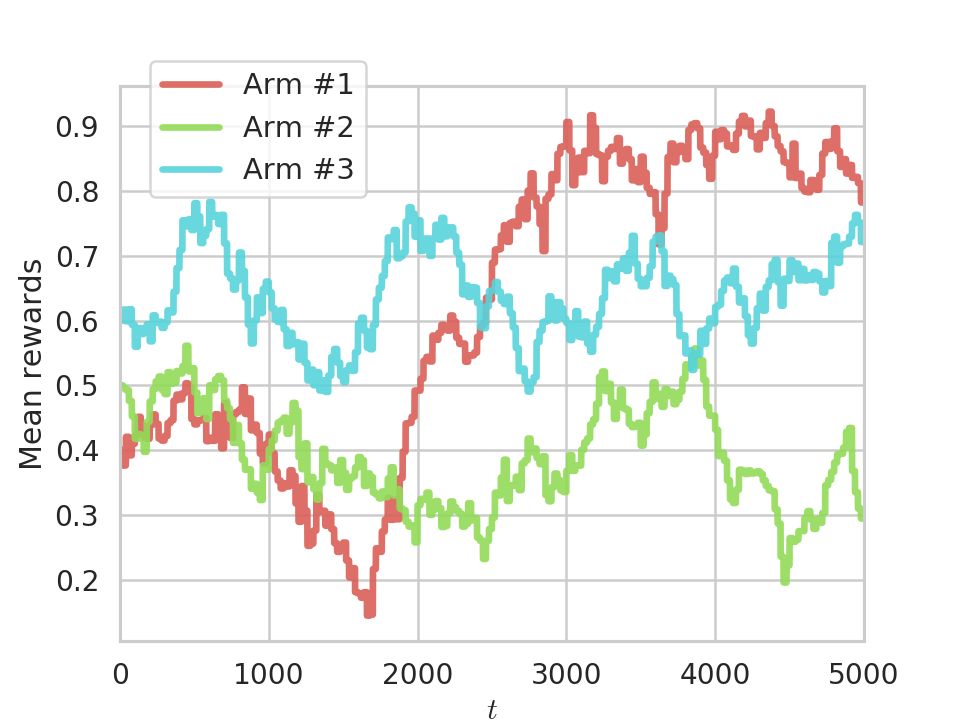}
		\caption{Environment $1$}\label{1rst}
	\end{subfigure}%    <-- % added here
	\hfill %% useful if width of each figure is less the .5\textwidth
	\begin{subfigure}[b]{0.24\textwidth}
		\centering
		\includegraphics[width=\textwidth]{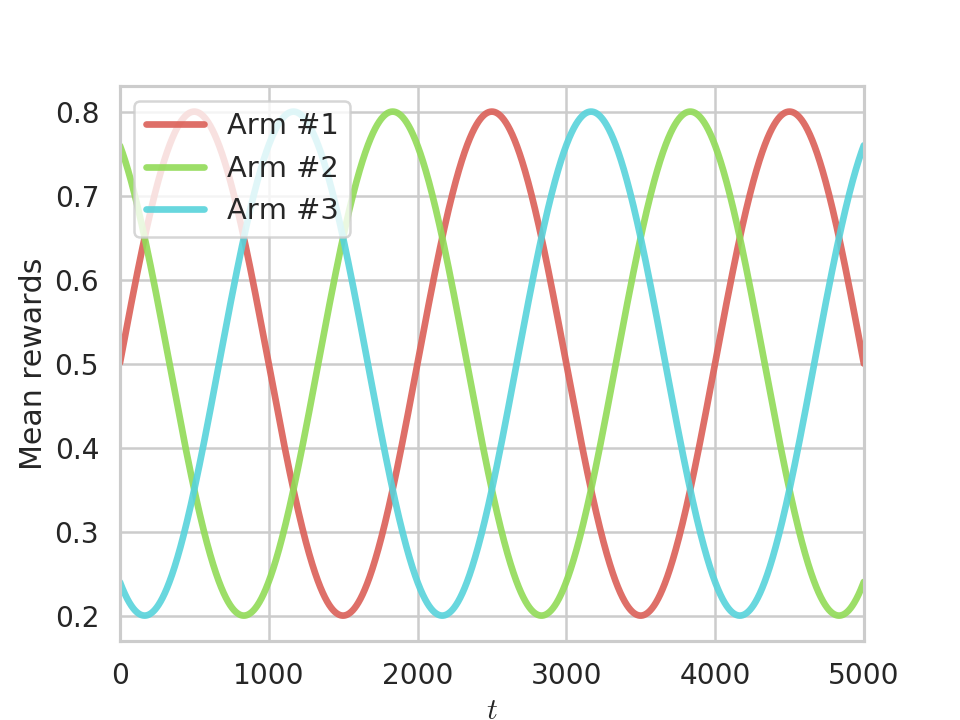}
		\caption{Environment $2$}\label{2nd}
	\end{subfigure}
	\hfill %% useful if width of each figure is less the .5\textwidth
	\begin{subfigure}[b]{0.24\textwidth}
		\centering
		\includegraphics[width=\textwidth]{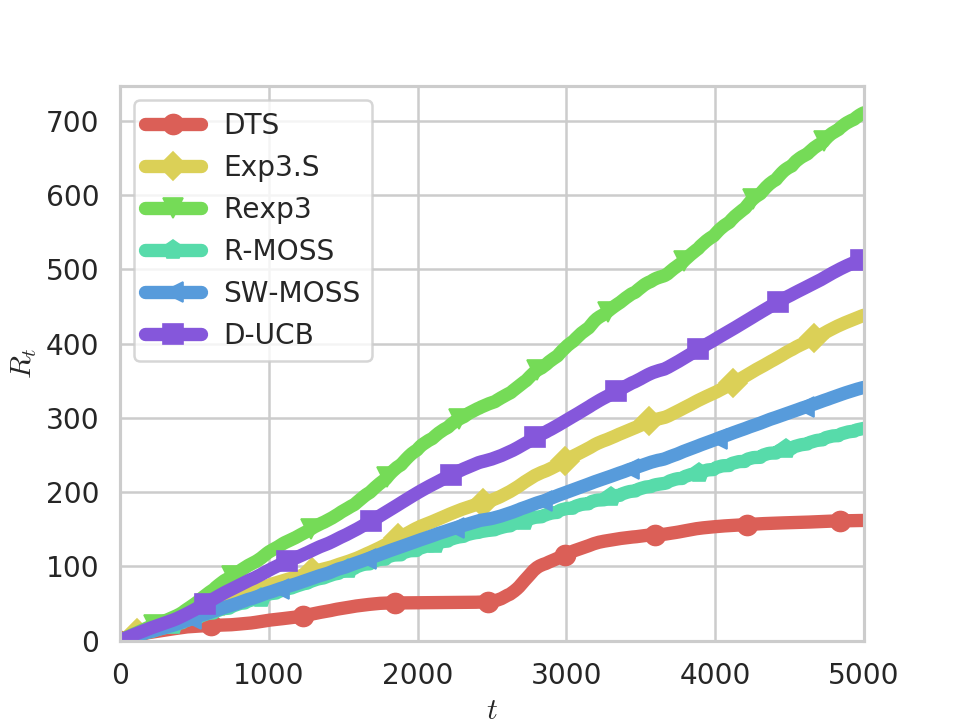}
		\caption{Regrets for environment $1$}\label{reg1}
	\end{subfigure}
	\hfill %% useful if width of each figure is less the .5\textwidth
	\begin{subfigure}[b]{0.24\textwidth}
		\centering
		\includegraphics[width=\textwidth]{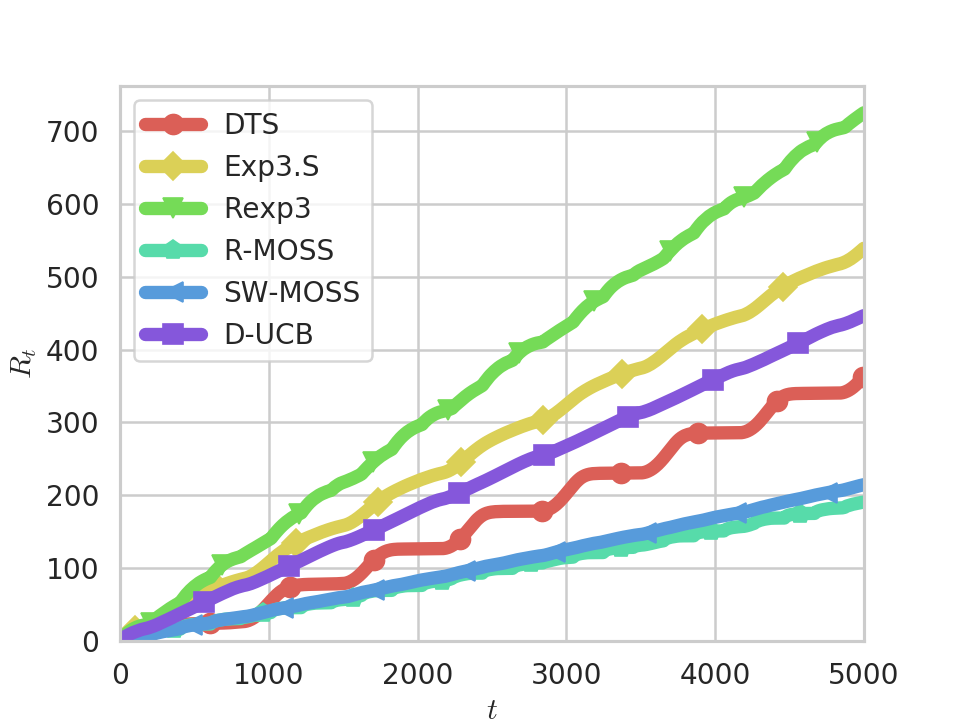}
		\caption{Regrets for environment $2$}\label{reg2}
	\end{subfigure}
	
	\caption{Comparison of different policies.}\label{key}
\end{figure}

The growths of regret in Figs.~\ref{reg1} and~\ref{reg2} show that UCB based policies (R-MOSS, SW-MOSS, and D-UCB) maintain their superior performance against adversarial bandit based policies (Rexp$3$ and Exp$3$.S)  for stochastic bandits even in nonstationary settings, especially for R-MOSS and SW-MOSS. Besides, DTS outperforms other polices when the best arm does not switch. While each switch of the best arm seems to incur larger regret accumulation for DTS, which results in a lager regret compared with SW-MOSS and R-MOSS.

\subsection{Heavy-tailed Nonstationay Stochastic MAB Experiment}
Again we consider the $3$-armed bandit problem with sinusoidal mean rewards. In particular, for each arm $k \in \{1,2,3\}$,
\[\mu_t^k = 0.3 \sin\left(0.001 \pi t + 2k\pi/3\right), \quad  t\in \until{5000}.\]
Thus, the variation budget is $3$. Besides, mean reward is contaminated by additive sampling noise $\nu$, where $\abs{\nu}$ is a generalized Pareto random variable and the sign of $\nu$ has equal probability to be ``$+$" and ``$-$". So the probability distribution for $X_t^k$ is
\[f_t^{k}(x) = \frac{1}{2\sigma}\left(1 + \frac{\xi \abs{x-\mu_t^k}}{\sigma}\right)^{-\frac{1}{\xi} - 1} \, \text{for } x \in (-\infty, +\infty).\]
We select $\xi = 0.4$ and $\sigma=0.23$ such that Assumption \ref{ass: heavy-tailed} is satisfied. We select $a=1.1$ and $\zeta = 2.2$ for both R-RMOSS and SW-RMOSS such that condition $\psi(2\zeta/a) \geq 2a/\zeta$ is met. 

\begin{figure}[ht!]
	\centering
	\begin{subfigure}[b]{0.24\textwidth}
		\centering
		\includegraphics[width=\textwidth]{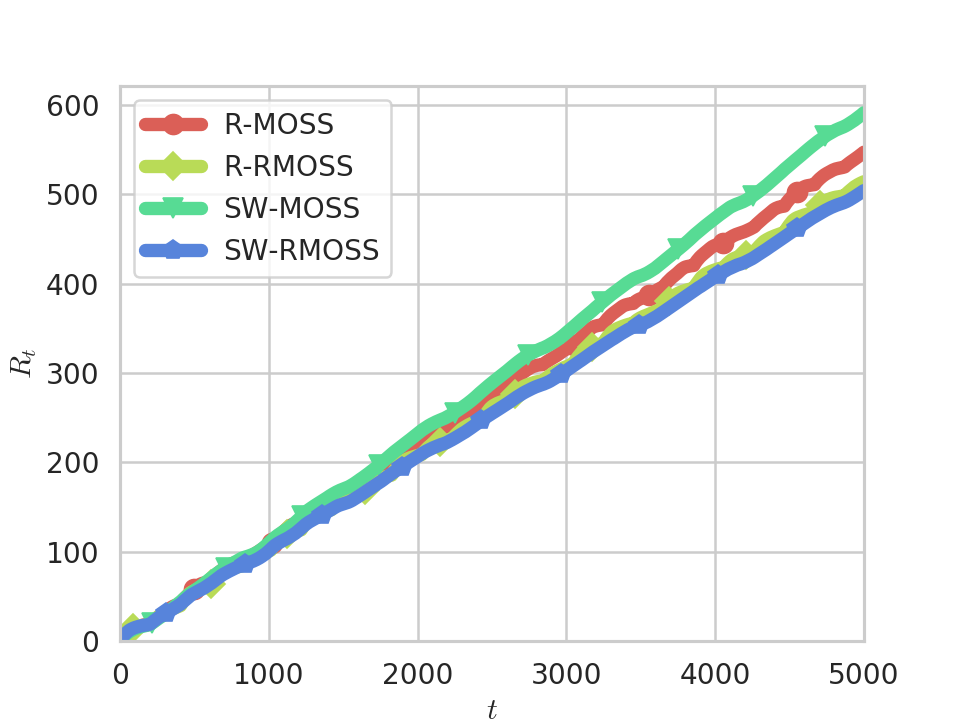}
		\caption{Regret}\label{regret}
	\end{subfigure}%    <-- % added here
	\hfill %% useful if width of each figure is less the .5\textwidth
	\begin{subfigure}[b]{0.24\textwidth}
		\centering
		\includegraphics[width=\textwidth]{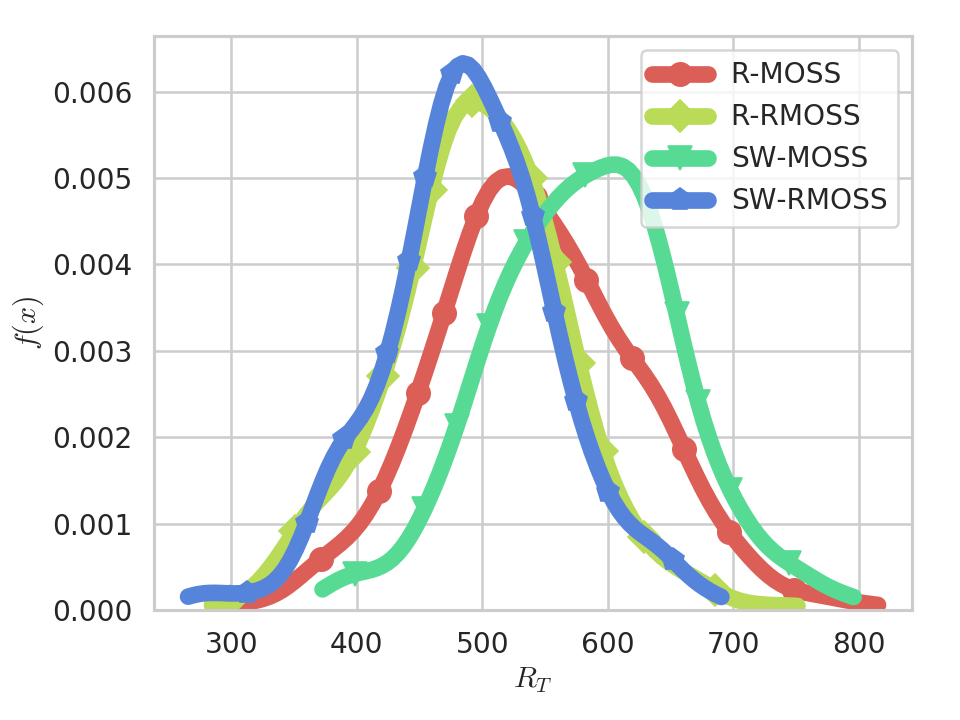}
		\caption{Histogram of $R_T$}\label{histogram}
	\end{subfigure}

	\caption{Performances with heavy-tailed rewards.}
\end{figure}

Fig.~\ref{regret} show RMOSS based polices and slightly outperform MOSS based polices in heavy-tailed settings. While by comparing the estimated histogram of $R_T$ for different policies in Fig.~\ref{histogram}, R-RMOSS and SW-RMOSS have a better consistency and a smaller possibility of a particular realization of the regret deviating significantly from the mean value. 

\section{Conclusion} \label{sec: conclusion}
We studied the general nonstationary stochastic MAB problem with variation budget and provided three UCB based policies for the problem. Our analysis showed that the proposed policies enjoy the worst-case regret that is within a constant factor of the minimax regret lower bound. Besides, the sub-Gaussian assumption on reward distributions is relaxed to define the nonstationary heavy-tailed MAB problem. We show the order optimal worst-case regret can be maintained by extending the previous policies to robust versions.
% We also broadly discuss the techniques to transform the policies from non-anytime to anytime.s

There are several possible avenues for future research. In this paper, we relied on passive methods to balance the remembering-versus-forgetting tradeoff. The general idea is to keep taking in new information and removing out-dated information. Parameter-free active approaches that adaptively detect and react to environment changes are promising alternatives and may result in better experimental performance. Also extensions from the single decision-maker to distributed multiple decision-makers is of interest. Another possible direction is the nonstaionary version of rested and restless bandits.

\bibliographystyle{IEEEtran}

\bibliography{mybib}

\end{document}